\documentclass[Afour,sageh,times]{sagej}

\usepackage{amssymb}
\usepackage{verbatim}
\usepackage{amsmath}
\usepackage{array}
\usepackage{graphicx}
\usepackage{wrapfig}
\usepackage{color}
\usepackage[titlenumbered,ruled,linesnumbered]{algorithm2e}
\usepackage{url}
\usepackage{subfig}
\usepackage{float}
\usepackage{appendix}
\usepackage{amsthm}
\newcommand{\D}{D}

\newcommand{\rad}{\ensuremath{\operatorname{rad}}}

\newcommand{\R}{\ensuremath{\mathbb{R}}}

\renewcommand{\SS}{\ensuremath{\mathcal{S}}}
\newcommand{\OO}{\ensuremath{\mathcal{O}}}

\newcommand{\cc}{\ensuremath{\mathcal{C}}}
\newcommand{\cfree}{\ensuremath{\mathcal{C}^{free}}}
\newcommand{\ccol}{\ensuremath{\mathcal{C}^{col}}}

\newcommand{\afree}{\ensuremath{a\mathcal{C}^{free}}}
\newcommand{\acol}{\ensuremath{a\mathcal{C}^{col}}}

\newcommand{\sfree}{\ensuremath{Sl^{free}}}
\newcommand{\scol}{\ensuremath{Sl^{col}}}

\newcommand{\asfree}{\ensuremath{aSl^{free}}}
\newcommand{\ascol}{\ensuremath{aSl^{col}}}

\DeclareMathOperator{\rot}{R}
\DeclareMathOperator{\dsign}{d_{s}}
\newtheorem{definition}{Definition}
\newtheorem{proposition}{Proposition}
\newtheorem{corollary}{Corollary}
\newtheorem{obs}{Observation}
\newtheorem{rmk}{Remark}

\newcommand\BibTeX{{\rmfamily B\kern-.05em \textsc{i\kern-.025em b}\kern-.08em
T\kern-.1667em\lower.7ex\hbox{E}\kern-.125emX}}

\def\volumeyear{2019}

\begin{document}

\runninghead{Free Space of Rigid Objects}

\title{Free Space of Rigid Objects: \\ Caging, Path Non-Existence, and Narrow Passage Detection}

\author{Anastasiia Varava*\affilnum{1}, J. Frederico Carvalho*\affilnum{1}, Danica Kragic\affilnum{1}, and Florian T. Pokorny\affilnum{1}}

\affiliation{*The first two authors contributed equally. \\
\affilnum{1}Robotics, Perception, and Learning,  KTH Royal Institute of Technology, \\ Sweden}

\corrauth{Anastasiia Varava}
\email{\{varava\}@kth.se}

\begin{abstract}
In this work we propose algorithms to explicitly construct a conservative estimate of the configuration spaces of rigid objects in 2D and 3D.
Our approach is able to detect compact path components and narrow passages in configuration space which are important for applications in robotic manipulation and path planning.
Moreover, as we demonstrate, they are also applicable to identification of molecular cages in chemistry.
Our algorithms are based on a decomposition of the resulting 3 and 6 dimensional configuration spaces into slices corresponding to a finite sample of fixed orientations in configuration space.
We utilize dual diagrams of unions of balls and uniform grids of orientations to approximate the configuration space.
We carry out experiments to evaluate the computational efficiency on a set of objects with different geometric features thus demonstrating that our approach is applicable to different object shapes.
We investigate the performance of our algorithm by computing increasingly fine-grained approximations of the object's configuration space.
A multithreaded implementation of our approach is shown to result in significant speed improvements.
%    In this paper, we present an approach towards approximating configuration spaces of 2D and 3D rigid objects. The approximation can be used to identify caging configurations and establish path non-existence between given pairs of configurations. We prove correctness and analyse completeness of our approach. Using dual diagrams of unions of balls and uniform grids on $SO(3)$, we provide a way to approximate a 6D configuration space of a rigid object. 
%Our approach is applicable to 3D objects independently of their global geometric features.
% Depending on the desired level of guaranteed approximation accuracy, the experiments with our single core implementation show runtime between $5-21$ s. and $463-1558$ s.
% Finally, we establish a connection between robotic caging and molecular caging from organic chemistry, and demonstrate that our approach is applicable to 3D molecular models.
\end{abstract}
%\keywords{caging, path non-existence, computational geometry}

\maketitle

\section{Introduction}
\label{intro}

A basic question one may ask about a rigid object is to what extent it
can be moved relative to other rigidly fixed objects in the environment. In robotic
manipulation, this has lead to the notion of a \emph{cage}. An object is
considered caged by rigid fixtures, or a robotic manipulator, if it cannot be moved
arbitrarily far from its initial configuration. In terms of the configuration space of
the rigid object, this is equivalent to being able to answer
whether an object is located in a configuration contained in a bounded path
component of its configuration space. Similarly, in fields such as chemistry and
biology, this notion of a cage is a useful basic concept for predicting how molecules can
restrict each other's mobility which has important applications to drug delivery and related
problems~\citep{Mitra2013,rother}.

The main challenge in caging verification is that the configuration space of a rigid
object in 3D is in general a 6 dimensional subset of SE(3). For this reason, explicit reconstruction
of the configuration space in terms of higher-dimensional analogues of discretization
techniques such as voxel-grids and triangulations has been considered a computationally
infeasible approach to this problem~\citep{makita-survey}. Past work has instead focused on analyzing caging
configurations either for 2D objects only~\cite{mccarthy}, or for specific 3D objects with special
geometric properties, such as handles and narrow parts~\citep{pokorny, varava}, which allow one to avoid modelling the configuration space directly.

\begin{figure*}[htb!]
\centering
\includegraphics[width=01\textwidth]{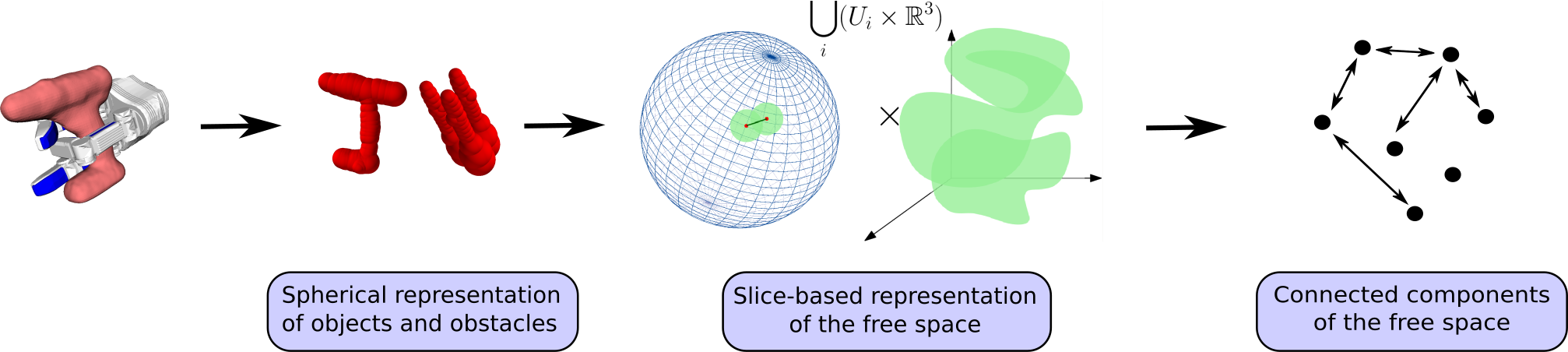}
\caption{\label{fig::teaser} Diagram of our method. We approximate the collision space of an object by choosing a finite set of fixed object's orientations and considering the corresponding slices of the collision space to $\R^n$ ($n \in \{2, 3\}$). From the collision space slices we compute approximations of the free space slices. Finally, we analyze the connectivity between neighboring slices to get an approximation of the connected components of the entire free space.}
\end{figure*}

We study caging as a special case of proving path non-existence between a pair of configurations.
To show that two configurations are disconnected, we construct an approximation of the object's collision space. Intuitively, we construct a set of slices of the object's collision space to subspaces corresponding to fixed orientations of the object, see Fig.~\ref{fig::teaser}. We then compute the free space approximation as the complement to the collision space of each slice. By construction, our collision space approximation is a proper subset of the real collision space, which implies that when our algorithm reports that the two configurations are not path-connected, then there is no path between them in the real free space. However, for the same reason, our algorithm is not guaranteed to find all  possible caging configurations, since we do not reconstruct the entire collision space.

The key contribution of the work we present here is to show that \emph{it is possible to compute explicit approximations of configuration
spaces of generic rigid bodies in 3D relatively efficiently,} while maintaining provability guarantees with respect to
reasoning about caging and, more generally, path existence. 
Our technique for constructing such an approximation
 is based on the following key insights:
\begin{itemize}
    \item \textbf{Representation:} We utilize a union-of-balls based object and obstacles representation
        that allows one to utilize dual diagrams to approximate the free space in a
        provably correct manner.
    \item \textbf{Slicing:} We use an approximate configuration space decomposition based on
        locally fixed orientations.
    \item \textbf{Object shrinking:} We use a subset of the object (its $\varepsilon-$core), which makes it possible to construct a sliced-based approximation of its configuration space.
    \item \textbf{Parallelization:} Our approach allows for parallel slice computation,
        leveraging modern CPU architectures.
\end{itemize}

The presented work constitutes an extended version of our initial work at WAFR'18~\citep{varavaWAFR}. It features an revised introduction and algorithms description, as well as extended experimental results, a parallelized version of the presented algorithm and its evaluation on 3D models of real objects.

\section{Related Work}\label{related-work}

In manipulation, caging can be considered as an alternative to a force-closure grasp~\citep{makita,makita2,pokorny,varava}, as well as an intermediate step on the way towards a form-closure grasp~\citep{rodriguez}. Unlike classical grasping, caging can be formulated as a purely geometric problem, and therefore one can derive sufficient conditions for an object to be caged. To prove that a rigid object is caged, it is enough to prove this for any subset (part) of the object. This allows one to consider caging a subset of the object instead of the whole object, and makes caging robust to noise and uncertainties arising from shape reconstruction and position detection.
%=====================caging

The notion of a planar cage was initially introduced by \cite{kuperberg} as a set of $n$ points lying in the complement of a polygon and preventing it from escaping arbitrarily far from its initial position. In robotics, it was subsequently studied in the context of point-based caging in 2D by \cite{rimon,sudsang_polygons,vahedi}, and others.  A similar approach has also been adopted for caging 3D objects. For instance, \cite{sudsang_polytopes} proposed an algorithm for computing all two-finger cages for non-convex polytopes. \cite{pereira} and \cite{wang}  present a set of 2D caging-based algorithms enabling a group of mobile robots to cooperatively drag a trapped object to the desired goal.

In the above mentioned works fingertips are represented as points or spheres. Later, more complex shapes of caging tools were taken into account by \cite{pokorny,stork2013b,varava,makita,makita2}. In these works, sufficient conditions for caging were derived for objects with particular shape features. \cite{makapunyo} proposed a heuristic metric for partial caging based on the length and curvature of escape paths generated by a motion planner. The authors suggested that configurations that allow only rare escape motions may be used as cages in practice.

We address caging as a special case of the path non-existence problem: an object is caged if there is no path leading it to an unbounded path-connected component.
The problem of proving path non-existence has been addressed by \cite{basch} in the context of motion planning, motivated by the fact that most modern sampling-based planning algorithms do not guarantee that two configurations are disconnected, and rely on stopping heuristics in such situations~\citep{latombe}.
\cite{basch} provide an algorithm to prove that two configurations are disconnected when the object is `too big' or `too
long' to pass through a `gate' between them. There are also some related results on approximating configuration spaces of 2D objects. \cite{zhang} use approximate cell decomposition and prove path non-existence for 2D rigid objects. They decompose a configuration space into a set of cells and for each cell decide if it lies in the collision space. \cite{mccarthy} propose a related approach. There, they randomly sample the configuration space of a planar rigid object and reconstruct its approximation as an alpha complex. They later use it to check the connectivity between pairs of configurations. This approach has been later extended to planar energy-bounded caging by \cite{mahler}.

The problem of explicit construction (either exact or approximate) of configuration spaces has been studied for several decades in the context of motion planning, and a summary of early results can be found in the survey by~\cite{wise}. \cite{lozano-perez} introduced the idea of slicing along the rotational axis. To connect two consecutive slices, the authors proposed to use the area swept by the robot rotating between two consecutive orientation values. \cite{zhu} extended this idea and used both outer and inner swept areas to construct a subset and a superset of the collision space of polygonal robots. The outer and inner swept areas are represented as generalized polygons defined as the union and intersection of all polygons representing robot's shape rotating in a certain interval of orientation values, respectively. 
Several recent works propose methods for exact computation of configuration spaces of planar objects~\citep{behar,milenkovic}. \cite{behar} proposed a method towards exact computation of the boundary of the collision space. \cite{milenkovic} explicitly compute the free space for complete motion planning. 

Thus, several approaches to representing configuration spaces of 2D objects, both exact and approximate, have been proposed and successfully implemented in the past. The problem is however more difficult if we consider a 3D object, as its configuration space is 6-dimensional.  In the recent survey on caging by \cite{makita-survey}, the authors hypothesise that recovering a 6D configuration space and understanding caged subspaces is computationally infeasible. To the best of our knowledge, our paper presents the first practical and provably-correct method to approximate a 6D configuration space.

Our approximation is computed by decomposing the configuration space into a finite set of lower dimensional slices. Although the idea of slicing  is not novel and was introduced by \cite{lozano-perez}, recent advances in computational geometry and topology, as well as a significant increase in processing power, have made it possible to approximate a 6D configuration space on a common laptop.
We identify two main challenges to slicing a 6D configuration space of a rigid object:
how to quickly compute 3D slices of the free space, and
how to efficiently discretize the orientation space.
For slice approximation, our method relies on the fact that the collision space associated to a rigid  object with a fixed orientation and an obstacle represented as a union of balls is itself a union of balls. Then we use the dual diagram to a union of balls presented by~\cite{edelsbrunner-skin} as an approximation  of the free space of the slice. This way, we do not need to use generalized polygons, which makes previous approaches more difficult in 2D and very hard to generalize to 3D workspaces. To discretize $SO(3)$, we use the method by \cite{yershova}, which provides a uniform grid representation of the space. The confluence of these factors results in overcoming the dimensionality problem without losing necessary information about the topology of the configuration space, and achieving both practical feasibility and theoretical guarantees at the same time.

Finally,  our method does not require precise information about the shape of objects and obstacles, and the only requirement is that balls must be located strictly inside them, which makes our approach more robust to noisy and incomplete sensor data.

We implemented our algorithm to approximate 3D and 6D configuration spaces and verify that it has in practice a reasonable runtime on a single core of Intel Core i7 processor. We provide a parallel implementation which makes use of modern parallel CPU architectures and investigate the effect of parallelization on the runtime of the algorithm.

\section{Definitions and Notation}
\subsection{Objects and obstacles}
For the sake of generality, in this paper we use the terms `object' for objects and autonomous rigid robots (e.g., disc robots) moving in $n$-dimensional workspaces, where $n \in \{2, 3\}$. Similarly, we use the term `obstacle' for everything that restricts mobility of the object -- e.g., manipulators, walls, other rigidly fixed objects, etc.

When formally defining an object and a set of obstacles, we make a few mild assumptions to define a large class of shapes and include most of the regular objects in the real world. 
Since we want to represent both the object and the obstacles as a set of $n-$dimensional balls, we do not allow them to have `thin parts'. Formally, we assume that they can be represented as \emph{regular sets}~\citep{rodriguez-path}:

\begin{definition}
\label{dfn::regular_set}
A set $U$ is \emph{regular} if it is equal to the closure of its interior: 
$U = \operatorname{cl}(\operatorname{int}(U))$.
\end{definition}

In the above definition, the interior of $U$ is the largest open set contained in $U$, and the closure of $\operatorname{int}(U)$ is the smallest closed set containing $\operatorname{int}(U)$. In this paper, we assume that both the object and the obstacles  are regular sets.
Now, we define an object and a set of obstacles as follows: 

\begin{definition}\label{dfn::object}
A \emph{rigid object} is a regular compact connected non-empty subset of $\mathbb{R}^n$.
A \emph{set of obstacles}  is a regular compact  non-empty subset of $\mathbb{R}^n$.
\end{definition}

We approximate both the obstacles, $\SS$ and the object, $\OO$ as unions of balls which lie in their interior, that is, \mbox{$\SS=\{B_{R_1}(X_1),\ldots,B_{R_n}(X_n)\}$}  and \mbox{$\OO= \{B_{r_1}(Y_1), \ldots, B_{r_m}(Y_m)\}$} with radii $R_1, \ldots, R_n$ and $r_1, \ldots, r_m$ respectively. 

Let $\cc(\OO) = SE(n)$ denote the configuration space of the object. We define its \emph{collision space}\footnote{A theoretical analysis of different ways to define the free space can be found in~\cite{rodriguez-path}.} $\ccol(\OO)$ as the set of the objects configurations in which the object penetrates the obstacles:

\begin{definition}
\label{dfn::collision_space}
    $\ccol(\OO) = \{c \in \cc~|~[\operatorname{int} c(\OO)] \cap [\operatorname{int} \SS] \neq \emptyset\}$, where $c(\OO)$ denotes the object in a configuration $c$.
    The \emph{free space} $\cfree(\OO)$ is the complement of the collision space: $\cfree(\OO) = \cc(\OO) - \ccol(\OO)$.
\end{definition}

Note that this definition allows the object to be in contact with the obstacles.

\begin{definition}
Two configurations $c_1$ and $c_2$ are called \emph{path-connected} if there exists a continuous collision-free path $\gamma: [0, 1] \to \cfree(\OO)$ between them: $\gamma(0) = c_1$, $\gamma(1) = c_2$.
Two configurations are path-connected if and only if they belong to the same \emph{path-connected component}.
\end{definition}

To compute path-connected components of the free space, we decompose the free space into a set of $n$-dimensional slices.

\subsection{Slice-based representation of the C-space}
In our previous work~\citep{varava_2}, we suggested that configuration space decomposition may be a more computationally efficient alternative to its direct construction. We represent the configuration space of an object as a product $\cc(\OO) = \R^{n} \times SO(n)$, and consider a finite covering of $SO(n)$ by open sets (this is always possible, since $SO(n)$ is compact): $SO(n) = \bigcup_{i \in \{1,\ldots,s\}}U_i$. We recall the notion of a slice~\citep{varava_2}:

\begin{definition}
A \emph{slice} of the configuration space $\cc(\OO)$ of a rigid object, is a subset of $\cc(\OO)$ defined as follows:
$Sl_U(\OO) = \R^n \times U$,
where $U$ is an subset of $SO(n)$.
\end{definition} 

We denote a slice of the collision (free) space by $\scol_U(\OO)$ ($\sfree_U(\OO)$, respectively). For each slice, we construct an approximation $\ascol_U(\OO)$ of its collision space in such a way that our approximation lies inside the real collision space of the slice, $\ascol_U(\OO) \subset \scol_U(\OO) $.

This way, we approximate the entire collision space by a subset $\acol(\OO)$: 

{
\[
    \acol (\OO)= \left(\bigcup_{i \in \{1,\ldots,s\}} \ascol_{U_i}(\OO) \right)\subset \ccol(\OO)
\]
}

Now, we discuss how to construct slice approximations.

\subsection{An $\varepsilon-$core of the object}
First of all observe that by Def.~\ref{dfn::collision_space}, if a subset $\OO'$ of an object $\OO$ placed in configuration $c \in \cc(\OO)$ is in collision, then the entire object $\OO$ is in collision. Therefore, the collision space of $\OO$ is completely contained within the collision space of $\OO'$. This allows us to make the following observation:

\begin{obs}\label{obs::correctness}
    Consider an object $\OO$ and a set of obstacles $\SS$. Let $c_1, c_2 \in \cfree(\OO)$ be two collision-free configurations of the object. If there is no collision-free path between these configurations for its subset $\OO' \subset \OO$, then there is no collision-free path connecting these configurations for $\OO$.
\end{obs}
Therefore, if some subset $\OO'$ of $\OO$ in configuration $c$ is caged, then the entire object $\OO$ in the same configuration $c$ is caged.
This means that if we construct $\ascol_{U}$ in such a way that for any configuration $c \in \ascol_{U}$ there exists a subset $\OO'$ of $c(\OO)$ such that $\OO'$ is in collision, then $c(\OO)$ is also in collision. 
In our previous work~(\cite{varava_2}) we defined an $\varepsilon-$core of an object as follows:
\begin{definition}
The $\varepsilon$-core of an object $\OO$ is the set $\OO_{\varepsilon}$ comprising the points of $\OO$ which lie at a distance\footnote{By distance here we mean Euclidean distance in $\R^n$} of at least $\varepsilon$ from the boundary of $\OO$:
$\OO_{\varepsilon} = \{p\in\OO~|~d(p,\partial\OO)\geq\varepsilon\} $.
\end{definition}

\begin{figure}[htb!]
\centering
\includegraphics[width=0.5\linewidth]{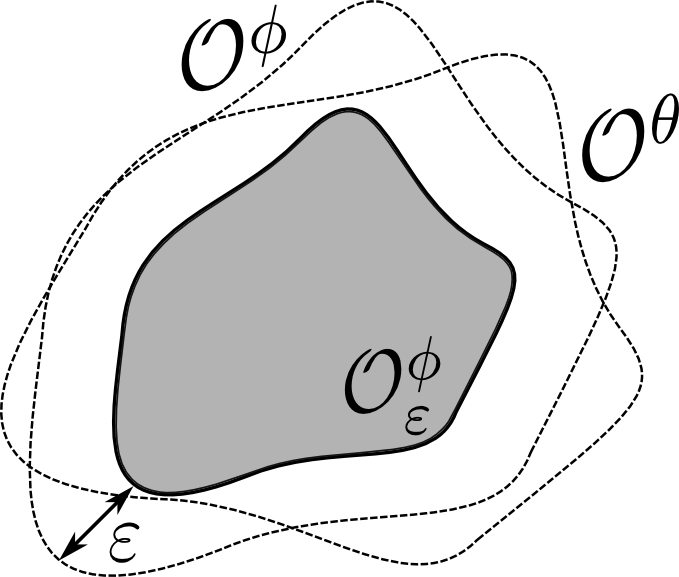}
\caption{\label{rotated}An $\varepsilon-$core remains inside the object when we slightly rotate it}
\end{figure} 

Now, for an object $\OO$ and its $\varepsilon$-core $\OO_{\varepsilon}$, we write $\OO^{\phi}$ and $\OO_{\varepsilon}^{\phi}$ respectively to mean that their orientation is fixed at $\phi \in SO(n)$. So, let $\ccol(\OO_{\varepsilon}^{\phi})$ denote the collision space of $\OO_\varepsilon$ with a fixed orientation $\phi$. Note that since the orientation is fixed, we can identify $\ccol(\OO_{\varepsilon}^{\phi})$ with a subset of $\R^n$.

In~\cite{varava_2}, we showed that for an object $\OO$, $\varepsilon>0$ and a fixed orientation $\phi\in SO(n)$ there exists a non-empty neighbourhood $U(\phi,\varepsilon)$ of $\phi$ such that for any $\theta \in U(\phi,\varepsilon)$, $\OO^{\phi}_{\varepsilon}$ is contained in $\OO^{\theta}$, see Fig.~\ref{rotated}.

In Sec.~\ref{orientation-discretization}, we address the problem of representing and discretizing the space of orientations $SO(n)$, and show how it is related to the notion of $\varepsilon-$core. 

%In practice, we approximate the collision space as follows: first, we discretize the orientation space $SO(n)$ and treat it as a set of orientation values $\{\phi_1, \ldots, \phi_s\}$. Given this discretization and the object, we compute the corresponding $\varepsilon > 0$ such that for any $\theta \in U(\phi,\varepsilon)$, $\OO^{\phi}_{\varepsilon}$ is contained in $\OO^{\theta}$.
%This way, we obtain a covering $U(\phi_1,\varepsilon), \ldots, U(\phi_s,\varepsilon)$ cover $SO(n)$, and  so that for any $\theta \in U(\phi_i,\varepsilon)$ we have $\OO^{\phi_i}_{\varepsilon} \subset \OO^{\theta}$.

%Finally, for each $\phi_i \in \{\phi_1, \ldots, \phi_s \}$, we compute collision space slice approximations $\ascol_{U(\phi_i,\varepsilon)}$ as the 3D collision space of $\OO^{\phi_i}_{\varepsilon}$, $\ascol_{U(\phi_i,\varepsilon)} = \ccol(\OO_{\varepsilon}^{\phi_i}) \times U(\phi_i,\varepsilon)$ which results in our approximation of the collision space: $\acol_\varepsilon(\OO) = \bigcup_{\phi_i}\ascol_{U(\phi_i,\varepsilon)}$.

\section{Existence of $\delta$-clearance paths}
For safety reasons, in path planning applications a path is often required to keep some minimum amount of clearance to obstacles. The notion of clearance of an escaping path can also be applied to caging: one can say that an object is partially caged if there exist escaping paths, but their clearance is small and therefore the object is unlikely to escape.

\begin{definition}
\label{def:delta-connect}
We say that two configurations are $\delta$-connected if and only if there exists a collision-free path of clearance at least $\varepsilon$ connecting these configurations.
\end{definition}

Consider a superset of our object $\OO$, defined as a set of points lying at most at distance $\delta$ from $\OO$, and let us call it a \emph{$\delta$-offset} of the object: $\OO_{+\delta} = \{p \in \R^n~|~d(p, \OO) \leq \delta \}$.
Equivalently to Def.~\ref{def:delta-connect}, we can say that two configurations $c_1$ and $c_2$ are $\delta$-connected in $\cfree(\OO)$ if and only if they are path-connected in $\cfree(\OO_{+\delta})$.

Consider now the following modification of our algorithm. If we enlarge the $\varepsilon-$core by a $\delta-$offset, then our rotated object will not be guaranteed to contain it anymore, but the distance between any point of the enlarged core that is located outside of the rotated object and the object itself will not exceed $\delta$. This means that if for this enlarged core our algorithm reports that two configurations are disconnected, then they either are disconnected in reality, or can be connected by a path of clearance at most $\delta$. Let us denote the resulting space approximation by $\afree_{\varepsilon, \delta}(\OO)$.

One application of this observation is narrow passage detection. One can use our free space approximation to identify narrow passages as follows.
If two configurations are disconnected in $\afree_{\varepsilon, \delta}(\OO)$, but connected in $\afree_{\varepsilon}(\OO)$, then they are connected by a narrow passage with clearance at most $\delta$. Our approximation then can be used to localize this passage, so that probabilistic path planning algorithm can sample in this location.

Furthermore, we can view $\delta$ as the level of accuracy of the approximation: assume we want to use a coarse discretization of the orientation space, and therefore the distance between adjacent orientations is large. This will require a larger $\varepsilon$, in which case some important information about the object's shape might not be captured by the $\varepsilon-$core. This might lead to a very conservative approximation of the free space. If now we consider  $\afree_{\varepsilon, \delta}(\OO)$, we might get more caging configurations. These results will be considered with respect to the used parameter $\delta$: if two configurations are disconnected in $\afree_{\varepsilon, \delta}(\OO)$, then the maximum possible clearance of a path connecting them in reality is at most $\delta$.

This leads us to a tradeoff between the desired accuracy of our approximation expressed in terms of clearance $\delta$, and the number of slices that we are willing to compute. The fewer slices we use, the larger $\delta$ we will need to consider.

% \begin{comment}
% We can also slightly modify Alg.~\ref{compute-slice-connectivity} to find potential narrow passages for each $\delta > 0$. In this case, instead of constructing a Delaunay triangulation, we can compute a nested family of alpha-complexes of the centers of the collision balls for \emph{all} positive alpha. As we mentioned before, this family is always finite. Increasing the $\alpha$ parameter corresponds to increasing the radius of the object under consideration, i.e., the family of alpha complexes approximates a nested family of collision space approximations of $\delta-$offsets of the object for all positive $\delta$. As we increase $\delta$, the topology of the free space of the slice will change, and the narrow passages will disappear. The corresponding values of $\delta$ reflect the width of the passages.
% \end{comment}

\section{Discretization of $SO(n)$}
\label{orientation-discretization}
Elements of $SO(n)$ can be seen as parametrizing rotations in $\R^n$, and for any $q \in SO(n)$ we define $\rot_q$ as the associated rotation. The notion of \emph{displacement} of a point after applying a rotation helps us to understand how the size of $\varepsilon-$core is related to the discretization of $SO(n)$:

\begin{definition}
Let $\D(\rot_q)$ denote the maximal \emph{displacement} of any point $p \in \OO$ after applying $\rot_q$, i.e.\ $\D (\rot_q) = \max_{p \in \OO}(d(p, \rot_q (p)))$, then $\OO_{\varepsilon} \subset \rot_q(\OO)$ if $\D (\rot_q) < \varepsilon$.
\end{definition}

Our goal now is to derive upper bounds for maximum displacement of any point in the object to make sure that the $\varepsilon-$core always remains inside the object when it is being rotated between different orientations belonging to the same slice, and how to efficiently discretize the orientation space.

\subsection{Displacement in 2D}

In our previous work (\cite{varava_2}), we derived the following upper bound for the displacement of a two-dimensional object:
\[
\D (\rot_q) \leq 2|\sin(q/2)|\cdot \rad(\OO),
\]
assuming that we rotate the object around its geometric center, $\rad(\OO)$ denotes the maximum distance from it to any point of the object, and $q$ is the rotation angle.

In the two-dimensional case, discretization of the rotation space is simple: given a desired number of slices, we obtain the displacement $\D(R_{q})$ induced by rotation between two neighboring orientations, and compute a set of orientation samples $\{\phi_1 = 0, \phi_2 = 2q,\ldots, \phi_{s} = 2(s-1)q\}$, where $s = \lceil \pi / q \rceil$. Then, we choose the $\varepsilon > \D (\rot_q)$. This gives us a covering $\{U(\phi_1,\varepsilon),\ldots, U(\phi_s,\varepsilon)\}$ of $SO(2)$, where for each $i \in \{1,\ldots, s\}$ we define $U(\phi_i,\varepsilon) = [\phi_i - q, \phi_i + q]$.\footnote{This is a cover by closed sets, but given $q' > q$ satisfying $\D(R_{q}) < \varepsilon$ we can use instead $U(\phi_i,\varepsilon) = (\phi_i - q', \phi_i + q')$ which results in the same graph.}.

\subsection{Displacement in 3D }

We now discuss the three-dimensional case. Similarly to the previous case, our goal is to cover $SO(3)$ with balls of fixed radius. To parametrize $SO(3)$ we use unit quaternions. For simplicity, we extend the real and imaginary part notation from complex numbers to quaternions, where $\Re q$ and $\Im q$ denote the real and ``imaginary'' parts of the quaternion $q$. Further, we identify $\Im q = q_i i + q_j j + q_k k$ with the vector $(q_i,q_j,q_k)\in\R^3$; and we write $\bar q$, and  $|q|$ to mean the conjugate $\Re q - \Im q$ and the norm $\sqrt{q \bar q}$, respectively.

A unit quaternion $q$ defines a rotation $R_q$ as the rotation of angle $\theta_q = 2\cos^{-1}(\Re q)$ around axis $w_q = \frac{\Im q}{|\Im q|}$. This allows one to calculate the displacement of the rotation $\D(q) = \D(R_q)$ as:
\[
\frac{\D(q)}{\rad(\OO)} = 2\sin(\frac{\theta_q}{2}) = 2\sin(\cos^{-1}(\Re q)) = 2 |\Im q|
\]

We use the \emph{angular distance}~(\cite{yershova}) to define the distance between two orientations:
\[\rho(x,y) = \arccos(|\langle x,y \rangle|),\]
where $x$ and $y$ are two elements of $SO(3)$ represented as unit quaternions which are regarded as vectors in $\R^4$, and $\langle x,y\rangle$ is their inner product.
We define the angle distance from a point to a set $S\subseteq SO(3)$ in the usual way as 
$$\rho(S,x) = \min_{y\in S} \rho(y,x)$$
 
 \cite{yershova} provide a deterministic sampling scheme to minimize the \emph{dispersion} 
 $$\Delta(S) = \max_{x\in SO(3)} \rho(S,x).$$
Intuitively, the dispersion of a set $\Delta(S)$ determines how far a point can be from $S$, and in this way it determines how well the set $S$ covers $SO(3)$.
Now, assume we are given a set of samples $S\subseteq SO(3)$ such that $ \Delta(S) < \Delta$.
Then for any point $p\in SO(3)$ denote $\D(p\bar S) = \max_{q\in S} \D(p\bar q)$, we want to show that in these conditions, there exists some small $\varepsilon$ such that $\max_{p \in SO(3)}\D(p\bar S) < \varepsilon$.
This would imply that there exists some $\Delta' > \Delta$ such that if we take $U = \{ q\in SO(3) ~|~ \rho(1,q) < \Delta' \}$, then denoting $U_p = \{qp~|~q \in U\}$, the family ${\{U_p\}}_{p\in S}$ fully covers $SO(3)$ and for any $q\in U_p$ satisfies $\D(q\bar p) < \varepsilon$.

Now, Proposition~\ref{prop:displacement-norm} allows us to establish the relation between the distance between two quaternions and displacement associated to the rotation between them.
\begin{proposition}\label{prop:displacement-norm}
	Given two unit quaternions $p,q$, the following equation holds:
    \begin{equation}
    \D(p\bar q) = 2\sin(\rho(p,q))\rad(\OO).
    \end{equation}
\end{proposition}
The proof of Proposition~\ref{prop:displacement-norm} can be found in appendix.

This means that if we want to cover $SO(3)$ with patches $U_i$ centered at a point $p_i$ such that $\D( p_i\bar q)$ for any $q\in U_i$ is smaller than some $\varepsilon$, we can use any deterministic sampling scheme on $SO(3)$ (e.g.~\citep{yershova}) to obtain a set $S$ with dispersion $\Delta(S) < \arcsin(\frac{\varepsilon}{2\rad(\OO)})$. Finally, by considering patches of the form $U(s,\varepsilon) = \{p \in SO(3) | \rho(s,p) < \Delta \}$, we obtain the corollary:

\begin{corollary}
If $S \subseteq SO(3)$ has a dispersion $\Delta(S)<\arcsin(\frac{\varepsilon}{2\rad(\OO)})$ then the family of patches $\{ U(s,\varepsilon) ~|~ s \in S \}$ forms a cover of $SO(3)$.
\end{corollary}

Given such a cover of $SO(3)$, recall that we want to approximately reconstruct the full free space of the object $\cfree$ as the union of slices $\asfree_{U(s,\varepsilon)}$. This requires us to test whether the orientation components of two slices overlap. To make this efficient, we create a \emph{slice adjacency graph}, which is simply a graph with vertices $S$ and edges $(s,s')$ if $U(s,\varepsilon)\cap U(s',\varepsilon) \neq \emptyset$.

To compute the graph adjacency, we note that if two slices $U(s,\varepsilon), U(s',\varepsilon)$ overlap, there must exist some $p\in SO(3)$ such that $\rho(p,s),\rho(p,s')<\Delta$, which implies $\rho(s,s') < \rho(s,p) + \rho(s',p) < 2\Delta$. This is leveraged in Alg.~\ref{alg:adjacency-graph}.

%==========================================================================
\begin{algorithm}[hbt!]
    \caption{\label{alg:adjacency-graph}%
        ComputeAdjacentOrientations%
    }
    \SetKwInOut{Input}{input}
    \SetKwInOut{Output}{output}
    \Input{$S$ --- a set of points in $SO(3)$. \\
	    $\Delta$ --- dispersion of $S$.}
    \Output{$G$  --- a patch adjacency graph.}
    $V \gets S$\\
    $T \gets KDTree(V)$\\
    $E \gets \{\}$\\
    \For{$p \in V$}{%
        $P\gets T.query(p,dist\leq 2\sin(\Delta))\setminus\{p\}$\label{l:query} \\
        $P\gets P \cup T.query(-p,dist\leq 2\sin(\Delta))$\label{l:query2}\\
        \For{$q \in L$}{%
            add $(p,q)$ to $E$\\
        }
    }
    \Return{$G=(V,E)$}\\
\end{algorithm}
%==========================================================================

The algorithm starts by setting $V = S$, as this is the set of vertices in the graph, and we put these vertices in a $KDTree$ in order to quickly perform nearest neighbor queries.
Now, to compute the set of edges $E$, we locate for each $p\in S$ the points at an angle distance smaller than $2\Delta${ }\footnote{Since the $KDTree$ $T$ uses the Euclidean distance in $\R^4$ we employ the formula $\| p - q \| = 2 \sin(\frac{\rho(p,q)}{2}) $.} in line~\ref{l:query}.
Finally, in line~\ref{l:query2} we also add edges to the points at an angle distance smaller than $2\Delta$ of $-p$, as both $p$ and $-p$ represent the same orientation.

\subsection{Different resolutions and dispersion estimation}

\begin{figure*}[htb!]
\centering
\includegraphics[height=4cm]{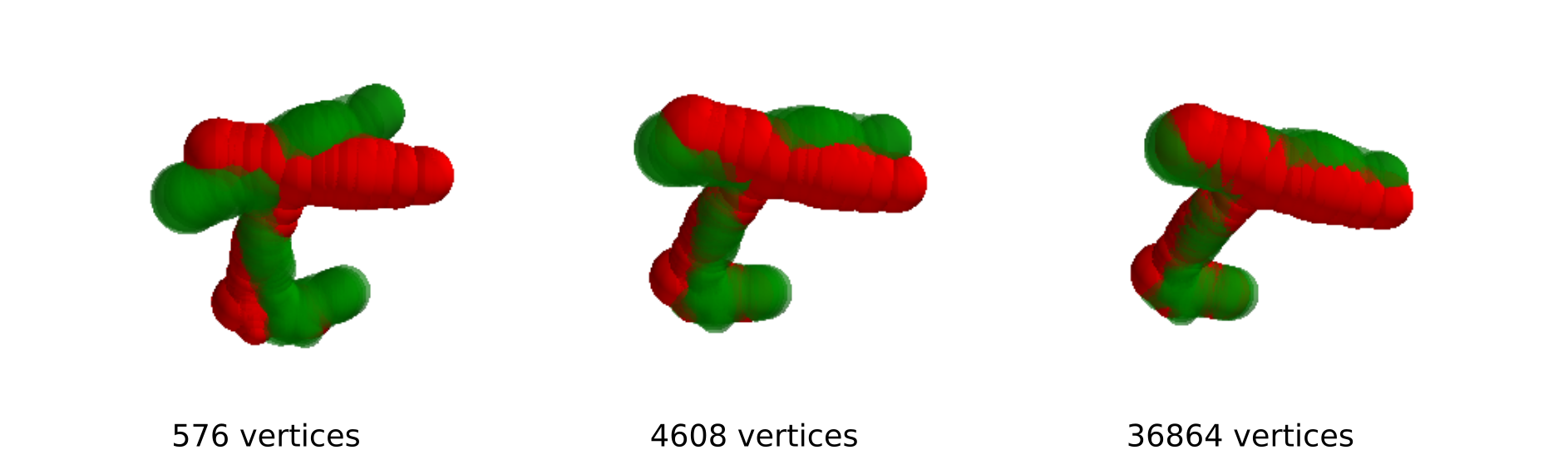}
\caption{ In this figure we show an approximation of a drill as a collection of balls around its medial axis. This representation is visualized at two distinct orientations which are adjacent in the graph over $SO(3)$. The labels underneath each subfigure show how many vertices the corresponding graph has. As expected the larger the grid, the denser the sampling of $SO(3)$ and the closer the two objects are to each other. On one extreme we observe that the overlap between the object in the first grid is very limited and therefore the obtained connectivity graph should not be taken to be very informative. On the other extreme, in the third grid, we see that the overlap between adjacent orientations is almost complete.}
\label{fig:different-rotations}
\end{figure*}

%==========================================================================
\begin{algorithm}[htb!]

    \caption{\label{alg:dispersion}%
        Approximating the dispersion of a set of points $S\subset SO(3)$%
    }
    \SetKwInOut{Input}{input}
    \SetKwInOut{Output}{output}
    \Input{$S$ --- Points in $SO(3)$.}
    \Output{$\Delta$ --- Approximate dispersion.}

    $\Delta \gets 0$\\
    $T \gets KDTree(V)$\\
    \For{$i  \in \{1,\ldots, N\}$ \label{ln:disp:loop-beg}}
    {
	    $v \gets \arg\max_{w\in SO(3)} \rho(S,w)$\label{line:maximize}\\
	    $w \gets T.query(v)$\label{line:query}\\
        $\Delta \gets \max(\Delta,2\sin^{-1}(\frac{\|w-v\|}{2})$\label{line:distance}$)$ \\
    \label{ln:disp:loop-end}
    }
    \Return{$\Delta$}
\end{algorithm}
%==========================================================================

One of the prerequisites of Alg.~\ref{alg:adjacency-graph} is the availability of an estimate of the dispersion of $\Delta(S)$. In our implementation we used the algorithm designed by~\cite{yershova} to compute $S$ where the authors provide a provable upper bound for dispersion. 

However, since this is an upper bound, and we want a tighter estimate of the dispersion to reduce the number of slices, in our implementation we employed a random sampling method to estimate the dispersion. % At each step we draw a uniform sample $p$ and compute its two nearest neighbors $q, q'$. We estimate the displacement at point $p$ to be half the (angle) distance between $q$ and $q'$. The final estimate is taken to be the maximum over the estimate for every sample taken.

Alg.~\ref{alg:dispersion} summarizes the procedure by which we approximate the dispersion. On each iteration of the main loop (lines~\ref{ln:disp:loop-beg}-\ref{ln:disp:loop-end}) it retrieves an element of $SO(3)$ which is locally the furthest from $S$ this is done by drawing a random sample followed by gradient ascent in line~\ref{line:maximize}. In line~\ref{line:query} the nearest element to $w$ of $S$ is retrieved (since the distance to this element will be the distance from $w$ to $S$), and finally in~\ref{line:distance} we use the fact that $\|w-v\| = 2\sin(\frac{\rho(w,v)}{2})$ to update the current value of the dispersion.

\subsection{The choice of $\varepsilon$ and $SO(3)$ discretization in practice}

Depending on the chosen discretization of the orientation space and its dispersion, the $\varepsilon$ value can be chosen such that it is greater than the estimated displacement:

\[
\varepsilon > 2 \sin\left(\dfrac{\Delta(S)}{2}\right) \rad(\OO)
\] 
 
To discretize $SO(3)$, we compute 3 different grids corresponding to different resolution levels using the software provided by \cite{yershova}. The first grid contains 576 vertices, and the respective dispersion estimate is 0.23. The second grid consists of 4608 vertices and its estimated dispersion is 0.10, and the third grid has 36864 vertices and a disperion estimate of 0.05. Fig.~\ref{fig:different-rotations} illustrates how much an object rotates between two neighboring orientations in each of these grids. In our experiments, we analyze how the performance of the algorithm is affected by the choice of a grid.

%==============================================================

\section{Free Space Approximation}
Let us now discuss how we connect the slices to construct an approximation of the entire free space, see Alg.~\ref{alg:connectivity-graph}.

Let $\mathcal{G}(\acol_\varepsilon(\OO)) = (V, E)$ be a graph approximating the free space.  The vertices of $\mathcal{G}$ correspond to the connected components $\{aC^i_1,\ldots, aC^i_{n_i}\}$ of each slice, $i \in \{1,\ldots,s\}$, and are denoted by $v = (aC, U)$, where $aC$ and $U$ are the corresponding component and orientation interval. Two vertices representing components $C_p \subset aSl_{U_i}^{free}$ and $C_q \subset aSl_{U_j}^{free}$, $i \neq j$, are connected by an edge if the object can directly move between them. For that, both the sets $U_i, U_j$, and $aC_q, aC_p$ must overlap: $U_i \cap U_j, aC_q \cap aC_p \neq \emptyset$.
 $\mathcal{G}(\acol_\varepsilon(\OO))$ approximates the free space of the object: if there is no path in $\mathcal{G}(\acol_\varepsilon(\OO))$ between the vertices associated to the path components of two configurations $c_1, c_2$, then they are disconnected in $\cfree(\OO)$.

We start by choosing one orientation and run a breadth-first search over the orientation grid $\mathcal{Q}$. When reaching a particular orientation, we construct the slice corresponding to it. 
In line~\ref{line:slice}, we compute the free space of a slice as explained in Alg.~\ref{compute-slice-connectivity}. 
In line~\ref{line:edges}, we check which connected components of adjacent slices overlap, and add edges between them, see Alg.~\ref{alg::compute-edges}.

\begin{algorithm}[htb!]
    \caption{\label{alg:connectivity-graph}ComputeConnectivityGraph}
    \SetKwInOut{Input}{input}
    \SetKwInOut{Output}{output}
    \Input{$O_{obst}$, $O_{obj}$ --- spherical representations of the obstacles and the object.\\
    $\delta$ --- clearance parameter \\
    $\mathcal{Q}$ --- a grid over $SO(3).$}
    \Output{$\mathcal{G}$ --- a connectivity graph of the free space.}
    $O_{\varepsilon} \gets ComputeEpsilonCore(O_{obj}, \delta, \mathcal{Q})$ \\
    $MarkedOrientations \gets \emptyset$\\
    $Queue \gets \emptyset$\\
    $q \gets \mathcal{Q}[0]$\\
    $MarkedOrientations.add(q)$\\ 
    $Slices[q] \gets ComputeSlice(O_{obst}, O_{\varepsilon}, q)$\\
    $Queue.add(q)$\\
    \label{line:while} \While{$Queue \neq \emptyset$} 
    {
    		$q_{cur} \gets Queue.dequeue()$\\
    		$CurrentSlice \gets Slices[q_{cur}]$\\
    		\For{$q_{adj} \in \mathcal{Q}.adjacentOrientations(q_{cur})$}
    		{
    		
    			\If{$q_{cur} \notin MarkedOrientations$}
    			{
    				\If{$Slices[q_{adj}] = \emptyset$}
    				{
    					$Slices[q_{adj}] \gets ComputeSlice(O_{obst}, O_{obj}, q)$ \label{line:slice}\\
    					$Queue.enqueue(q_{cur})$\\
    					$MarkedOrientations.add(q_{cur})$\\
    				}
    				$AddEdges(Slices[q_{cur}], Slices[q_{adj}])$ \label{line:edges}\\
				
    			}
    			
    		}
    		$Slices[q_{cur}].reset()$\\
    }
   
    \Return{$\mathcal{G}$}\\
\end{algorithm}

In order to query connectivity, the slices approximations should be preserved. In our current implementation, each slice is deleted as soon as the connectivity check between it and the slices adjacent to it is performed, in order to optimize memory usage. In this case, one can save the slice approximation together with the resulting graph to disk, for later use by a querying algorithm.

%==========================================
\subsection{Construction of Slices}
Now, let us discuss how we approximate path-connected components of the free space of each slice, see Alg.~\ref{compute-slice-connectivity}. 
Given a set of obstacles, an object, and a particular orientation of its $\varepsilon-$core, we start by computing the collision space of the slice.

In~\cite{varava_2}, we derive the following representation for the collision space of  $\OO^{\phi}_{\varepsilon}$:
\[
	\ccol(\OO_{\varepsilon}^{\phi}) = \bigcup_{i,j} (B_{R_j + r_i - \varepsilon}(X_j - \overline{GY_i})),
\]
where $G$ is the origin chosen as the geometric center of the object, and $\overline{GY_i}$ denotes the vector from $G$ to $Y_i$.

Indeed, the object represented as a union of balls collides with the obstacles if at least one of the balls is in collision, so the collision space of the object is a union of the collision spaces of the balls shifted with respect to the position of the balls centers:
\[
\ccol(\OO_{\varepsilon}^{\phi}) = \bigcup_{i \in \{1 ... m\}} \mathcal{C}^{col}(B_{r_i - \varepsilon}(Y_i)) - \overline{GY_i}
\]

Now, each ball collides with the obstacles when the distance from the obstacles to its center is not greater than the radius of the ball, so the collision space of a single ball of radius $r_i - \varepsilon$ can be written as:

$$
\{x \in \mathbb{R}^d | \operatorname{d}(x, S) \leq r_i - \varepsilon\}  = \bigcup_{j \in \{1 .. n\}} B_{R_j + r_i - \varepsilon}(X_j)
$$

%============================================
\begin{algorithm}[htb!]
    \SetKwInOut{Input}{input}
    \SetKwInOut{Output}{output}
  \Input{
	  $O_{obst}, O_{\varepsilon}$ --- spherical representations of the obstacles and the $\varepsilon-$core.\\
	  $q$ --- orientation.
  }
  \Output{
	  $aSl^{free}$ --- A set of connected components of the slice corresponding to $q$
	  }
  $\ccol(O_{\varepsilon}) \gets$ Collision-Space($O_{obst}$, $O_{obj}$, $q$) \\
  $V(\ccol(O_{\varepsilon})) \leftarrow $ Weighted-Voronoi-Diagram$(\ccol(O_{\varepsilon}))$\\
  $\afree \leftarrow$ Dual-Diagram($V(\ccol(O_{\varepsilon}))$)  \\    
  $aC_{0} \leftarrow$ Compute-Infinite-Component()\\
  $i \leftarrow 0$ \\
  \ForEach{$ball \in \afree$ }{
                \If{Connected-Component($ball$) $= \emptyset$}
                {$i \leftarrow i + 1$\\
                $aC_{i} \leftarrow$ Compute-Component($ball$)
                }
    }
    \Return{$aSl^{free} = \{aC_0,\ldots, aC_n\}$}
\caption{\label{compute-slice-connectivity} 
 ComputeSlice}
\end{algorithm} 
%=============================================

%\subsection{Free space a slice}
\begin{figure}[htb!]
\centering
\includegraphics[width=4cm]{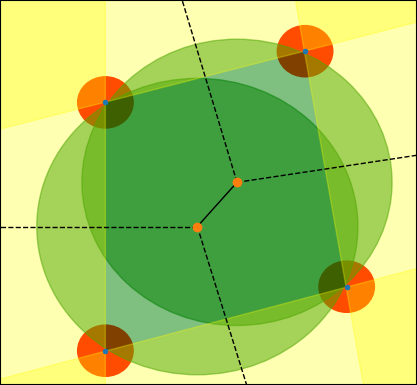}
    \caption{\label{toy-dual-diagram}  Green circles and yellow halfspaces (``infinite circles'') together represent a dual diagram of the red circles. Green circles are centered at the vertices of the Voronoi diagram; yellow halfspaces correspond to infinite edges of the diagram (dashed lines).}
\end{figure}

At the next step, we approximate the free space of the slice. For this, we  approximate the complement of the collision space by constructing the \emph{dual diagram}~(\cite{edelsbrunner-skin}) to the set  of balls representing the collision space.

%-------------------
\begin{figure*}[htb!]
\centering
\begin{tabular}{c@{\hspace{1em}}c@{\hspace{1em}}c@{\hspace{1em}}c}
{\includegraphics[height=3cm]{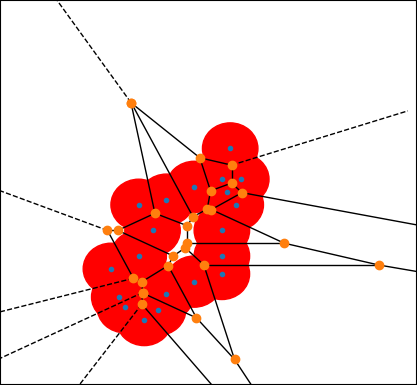}} &
{\includegraphics[height=3cm]{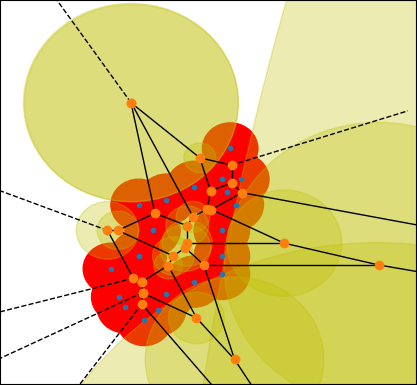}} &
{\includegraphics[height=3cm]{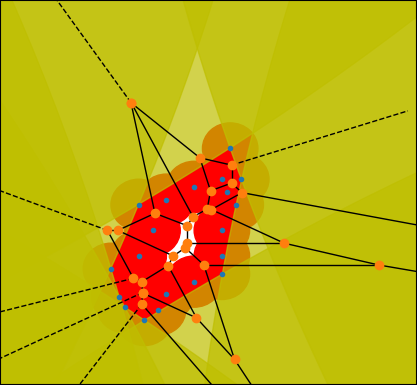}} &
{\includegraphics[height=3cm]{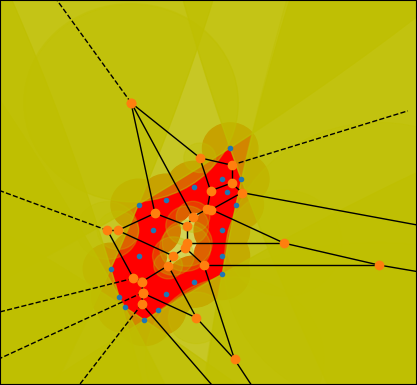}}
\end{tabular}
    \caption{\label{dual-diagram}  From left to right, the elements of the dual diagram approximating the free space. We start with building a Voronoi diagram (visualized in black); the second figure shows a set of finite orthogonal balls, corresponding to the regular edges of the diagram;  the third figure visualizes the ``infinite'' orthogonal balls corresponding to the infinite edges of the diagram; finally, the last figure depicts full diagram.  Red circles with blue centers represent the collision space, the yellow circles approximate the dual diagram. }
\end{figure*}

%-------------------
A dual diagram $Dual(\bigcup B_i)$ of a union of balls $\bigcup B_i$ is a finite collection of balls such that the complement of its interior is contained in and is homotopy equivalent to $\bigcup B_i$. It is convenient to approximate $\cfree(\OO_{\varepsilon}^{\phi})$ as a dual diagram of collision space balls for several reasons. First, balls are simple geometric primitives which make intersection checks trivial. Second, the complement of the dual diagram is guaranteed to lie strictly inside the collision space, which provides us with a conservative approximation of the free space. Finally, homotopy equivalence between its complement and the approximate collision space implies that our approximation preserves the connectivity of the free space that we want to approximate. Another advantage of a dual diagram is that it is easy to construct.

A weighted Voronoi diagram is a special case of a Voronoi diagram, where instead of Euclidean distance between points, a special distance function is used.  In our case, the weighted distance of a point $x \in \R^n$ from $B_{r_i}(z_i)$ is equal to $d_w(x,B_{r_i}(z_i))= ||x - z_i|| - r^2_i$. 

To construct a dual diagram of a union of balls $\bigcup_i B_{r_i}(z_i)$, we first construct their weighted Voronoi diagram, see Fig.~\ref{toy-dual-diagram}. For each vertex $y_j$ of the weighted Voronoi diagram, let $B_{q_j}(y_j)$ whose radius $q_j$ is equal to the square root of the weighted distance of $y_j$ to any of the four (three in the two-dimensional case) balls from $\bigcup_i B_{r_i}(z_i)$ generating $y_j$. Then, take each infinite edge of the Voronoi diagram, and add an degenerate ``infinitely large'' ball (a half-space) with center at the infinite end of this edge. The entire process of dual diagram construction can be seen on Fig.~\ref{dual-diagram}.

After constructing the dual diagram, we find pairs of overlapping balls in it and run depth-first search to identify connected components. It is important to note that each dual diagram always has exactly one unbounded connected component, representing the outside world.

%\subsection{Transitions between slices}
 
Finally, let us discuss how to understand whether free space approximations of neighboring slices overlap. In Alg.~\ref{alg::compute-edges}, to check whether two connected components in adjacent slices intersect (line~\ref{line:dointersect}), we recall that they are just finite unions of balls. Instead of computing all pairwise intersections of balls, we approximate each ball by its bounding box and then use the CGAL implementation of Zomorodian's algorithm (\cite{zomorodian}), which efficiently computes the intersection of two sequences of three-dimensional boxes. Every time it finds an overlapping pair of boxes, we check whether the respective balls also intersect.

\begin{algorithm}[htb!]

    \caption{\label{alg::compute-edges} AddEdges}
    \SetKwInOut{Input}{input}
    \SetKwInOut{Output}{output}
    \Input{$aSl_{U_i}^{free}, aSl_{U_j}^{free}$ --- free space approximations of two adjacent slices.}
    \Output{$E_{ij}$ --- a set of edges between the connected components of $aSl_{U_i}^{free}$ and $aSl_{U_j}^{free}$.}
    $E_{i,j} \gets \emptyset$\\
    
    \For{$aC_i \in ConnectedComponents(aSl_{U_i}^{free})$}
    {
    		\For{$aC_j \in ConnectedComponents(aSl_{U_j}^{free})$}
    		{    			
    			\If{$aC_i.isInfinite()$ AND $aC_j.isInfinite()$}
    			{
    				$E_{i,j}.addEdge(v(aC_i, U_i), v(aC_j, U_i))$\\
    			}
    			\Else
    			{
    				\If{$DoIntersect(aC_i, aC_j)$ \label{line:dointersect}} %\\
    				{
    					$E_{i,j}.addEdge(v(aC_i, U_i), v(aC_j, U_i))$\\
    				}
    			}
    		}
    }
    
    \Return{$E_{i,j}$}
\end{algorithm}

 \begin{rmk}\label{exact-graph}
     Recall that in each slice the $a\cfree(\OO^{\phi_i}_\varepsilon)$ are constructed as the dual of the collision space $\ccol(\OO^{\phi_i}_\varepsilon)$, which entails that $a\cfree(\OO^{\phi_i}_\varepsilon)$ has the same connectivity as $\cfree(\OO^{\phi_i}_\varepsilon)$.
     However, it also entails that any connected component of $a\cfree(\OO^{\phi_i}_\varepsilon)$ partially overlaps with the collision space $\ccol(\OO^{\phi_i}_\varepsilon)$.
     This means that for two connected components $C^i_j, C^{i'}_{j'}$ of adjacent slices which do not overlap, it may occur that the corresponding approximations $aC^i_j, aC^{i'}_{j'}$ do overlap. In this case the resulting graph $\mathcal{G}(\acol_\varepsilon(\OO))$ would contain an edge between the corresponding vetices. This effect can be mitigated by verifying whether the overlap between the approximations occurs within the collision space of both slices.  This can be done for example by covering the intersection $aC^i_j \cap aC^{i'}_{j'}$ with a union of balls and checking if it is contained inside the collision space $\ccol(\OO^{\phi_i}_\varepsilon)\cup \ccol(\OO^{\phi_{i'}}_\varepsilon)$.
 \end{rmk}

\section{Parallel implementation}
\label{sec:parallel}

To increase the performance of our algorithm, we designed and implemented a parallelized version that uses multiple threads to compute different slices simultaniously. Alg.~\ref{alg:parallel} describes the process.

Recall that in the single-thread version of the algorithm (see Alg.~\ref{alg:connectivity-graph}), we performed breadth-first search (BFS) over the orientation grid $\mathcal{Q}$. For each orientation, we computed the corresponding slice of the free space and its overlap with neighboring slices. The slice was deleted when all its neighbours were constructed and the connectivity check between them was performed.

A naive approach to parallelization would have a main thread call child threads on to process each new orientation in the queue, however since the orientations are enqueued in sequence it would increase the chances that different threads try to compete for computing the same slice data, necessitating threads to wait for each other to finish. Instead we opted for a simple scheme where each thread performs its own BFS from a different point in the tree taking care to lock access to resources every time shared data is accessed.

%Now, instead of choosing  asingle seed orientation and initializing the algorithm with it, we choose multiple seeds and perform the same procedure for them in parallel until several BFS trees reach each other. When this happens...{\bf TODO: complete}

\begin{algorithm}[htb!]
	\caption{\label{alg:parallel}ConnectivityGraphThread}
	\SetKwInOut{Input}{input}
	\SetKwInOut{Shared}{shared}
	\Input{
		$O_{obst}$, $O_{obj}$ --- spherical representations of the obstacles and the object.\\
		$\delta$ --- clearance parameter. \\
		$\mathcal{Q}$ --- a graph over $SO(3)$. \\
		$i$ --- an initial orientation.
	}
	\Shared{
		$\mathcal{M}$ --- an array containing a mutex per slice. \\
		$Slices$ --- an array of slices (one per orientation). \\
		$\mathcal{S}$ --- an integer array indicating slice status (0:unseen, 1:seen, 2:fully processed). \\
		$\mathcal{G}$ --- a connectivity graph of the free space \\
	}
	$O_{\varepsilon} \gets ComputeEpsilonCore(O_{obj}, \delta, \mathcal{Q})$ \\
	$Queue \gets \emptyset$\\
	$Queue.enqueue(i)$\\
	\While{$Queue \neq \emptyset$}{ \label{ln:thread:mainloop-begin}
		$q_{cur} \gets Queue.dequeue()$ \\
		$Lock(\mathcal{M}[q_{cur}])$ \\
		\If{ $\mathcal{S}[q_{cur}] = 0$ \label{ln:thread:Sqcur-0} }{
			$Slices[q_{cur}] \gets ComputeSlice(\mathcal{O}_{obst},\mathcal{O}_{\varepsilon},q_{cur}) $\label{ln:thread:proc-cur}\\
			$\mathcal{S}[q_{cur}] \gets  1$ \\
		}
		\If { $\mathcal{S} [q_{cur}] = 1$\label{ln:thread:Sqcur-1} } {
		$n \gets Size(\mathcal{Q}.adjacentOrientations(q_{cur}))$\\
		\For{$q_{adj} \in \mathcal{Q}.adjacentOrientations(q_{cur})$}{
			\label{ln:thread:cur-intersect-begin}
			$lock \gets TryLock(\mathcal{M}[q_{adj}])$ \label{ln:thread:lock-adj} \\ 
			\If{$lock.successful()$}{
				\If{$\mathcal{S}[q_{adj}] = 0$\label{ln:thread:ifunseen}}{
					$Slices[q_{adj}] \gets ComputeSlice(\mathcal{O}_{obst},\mathcal{O}_{\varepsilon},q_{adj})$ \\
					$\mathcal{S}[q_{adj}] \gets 1$\label{ln:thread:markseen} \\
					$Queue.enqueue(q_{adj})$
				}
				\If{$\mathcal{S}[q_{adj}] = 1$}{
				$AddEdges(Slices[q_{cur}],Slices[q_{adj}])$\label{ln:thread:add-edges}\\
				}
				$n\gets n-1$\\
				$Unlock(\mathcal{M}[q_{adj}])$
			}
		}
		\If {$ n = 0$ \label{ln:thread:qcur-all-done} } {
			$\mathcal{S}[q_{cur}] \gets  2$ \\
			$Slices[q_{cur}].reset()\label{ln:thread:freeslice}$ \\
		}
			\label{ln:thread:cur-intersect-end}
	}
	$Unlock(\mathcal{M}[q_{cur}])$
	\label{ln:thread:mainloop-end}
	}
	$NotifyParent()$
\end{algorithm}

Alg.~\ref{alg:parallel} describes the procedure followed by each thread. It assumes that all threads share an access to the array of slices and the connectivity graph, as well as the book-keeping data, i.e. an array of mutexes, and an array containing the status of each slice. The status array ($\mathcal{S}$) is used to decide what (if anything) should be computed about a given slice at a given moment, whereas the mutex array ($\mathcal{M}$) is used so two threads don't try to alter the same slice at the same time.

The main loop (lines~\ref{ln:thread:mainloop-begin}-\ref{ln:thread:mainloop-end}) proceeds as follows: it pops the first element of the queue $q_{cur}$ and locks its corresponding mutex. If it cannot obtain ownership of the mutex, it waits until this is freed by the other thread. It then verifies the status of the current orientation (lines~\ref{ln:thread:Sqcur-0} and~\ref{ln:thread:Sqcur-1}), if the slice associated to the orientation has yet to be computed it does so (line~\ref{ln:thread:proc-cur}) and if its edges have not yet been processed it goes on to process the edges to its neighboring slices (lines~\ref{ln:thread:cur-intersect-begin}-\ref{ln:thread:cur-intersect-end}). To compute the edges between the slices corresponding to the current orientation $q_{cur}$ and an adjacent orientation $q_{adj}$ the thread needs to also lock the mutex associated to $q_{adj}$ (line~\ref{ln:thread:lock-adj}) if it does so successfully then it proceeds to check the connection between $q_{cur}$ and $q_{adj}$, otherwise it checks the next adjacent orientation. When it has tried to calculate the edges to all the orientations adjacent to $q_{cur}$ it verifies if all of them were successfully calculated (line~\ref{ln:thread:qcur-all-done}), if so, the current orientation is marked as fully processed and the slice is deleted. Note also that the edges between $Slice[q_{cur}]$ and $Slice[q_{adj}]$ only get computed in line $\ref{ln:thread:add-edges}$ if $Slice[q_{adj}]$ has not been fully processed itself, as that would mean that the existing edges would have been processed in a previous step.

This procedure guarantees the absence of data corruption since in order to change the data associated to slices or the edges between them, the corresponding mutexes need to be locked first. Furthermore the algorithm is guaranteed to be lock-free given that:

\begin{itemize}
\item If two threads try to lock the same $q_{cur}$, the second one has to wait for the first one to finish before it can proceed.
\item If there is an edge $(u,v)$ of the orientation graph and one thread has $q_{cur} = v$ and a second thread has $q_{cur}=u$, then given that both instances only \emph{try} to lock the mutex of the corresponding to the other thread's orientation, they will not cause a deadlock.
\end{itemize}

Note also that each thread is guaranteed to terminate since orientations only get enqueued if their status is unseen ($\mathcal{S}[q_{adj}] = 0$ in line~\ref{ln:thread:ifunseen}) and this status is revoked immediately afterwards (line~\ref{ln:thread:markseen}) before it is enqueued. Furthermore since this occurs while the mutex associated to the corresponding slice is locked, no other thread is able to enqueue the same orientation. This guarantees that each orientation is only enqueued once.

Alg.~\ref{alg:parallel} is called by a threadpool as shown in Alg.~\ref{alg:parallel-caller}. The algorithm ComputeConnectivityGraphParallel works by setting up the data that must be shared between the several instances of ConnectivityGraphThread, and calling it starting from different orientations in $\mathcal{Q}$. Once all slices have been constructed the algorithm  proceeds by checking that they have all been completely processed, and otherwise checks the intersections with their remaining neighbors. This step is required in case the situation arises where two threads are processing neighboring orientations simultaneously. In this case they may fail to lock the other's mutex and therefore ignore the edges that may exist between their corresponding vertices. However since this failure implies that neither orientation is marked as fully processed, the corresponding slices are not freed in line~\ref{ln:thread:freeslice} of Alg.~\ref{alg:parallel}, and are therefore  available to be further processed in line~\ref{ln:caller:add-edges} of Alg.~\ref{alg:parallel-caller}.

\begin{algorithm}[htb!]
	\caption{\label{alg:parallel-caller}ComputeConnectivityGraphParallel}
	\SetKwInOut{Input}{input}
	\SetKwInOut{Output}{output}
	\SetKwInOut{Shared}{shared}
	\Input{$O_{obst}$, $O_{obj}$ --- spherical representations of the obstacles and the object \\
		$\delta$ --- clearance parameter. \\
	$\mathcal{Q}$ --- a graph over $SO(3)$ \\
	}
	\Output{
		$\mathcal{G}$ --- a connectivity graph of the free space
	}
	{\bf shared } $\mathcal{G} \gets (\emptyset,\emptyset)$\\
	{\bf shared } $\mathcal{M} \gets mutexArray[|\mathcal{Q}|]$ \\
	{\bf shared } $Slices \gets sliceArray[|\mathcal{Q}|]$ \\
	{\bf shared } $\mathcal{S} \gets intArray[|\mathcal{Q}|] $\\
	\While{ $\exists i : \mathcal{S}[i] = 0 $}{
		$q\gets SelectQuaternion(\mathcal{S})$\\
		$LaunchThread$ $ConnectivityGraphThread(O_{obst},O_{obj},\mathcal{Q},\delta,q)$\\
		$NThreads \gets NThreads-1$ \\
		\If{ $NThreads = 0$ }{
			$WaitForAnyChild()$\\
		}
	}
	$WaitForAllChildren()$\\
	\For{ $q_{cur} \gets \mathcal{Q}$ } {
		\If{ $\mathcal{S}[q_{cur}] = 1$ }{
			\For{$q_{adj} \in \mathcal{Q}.adjacentOrientations(q_{cur})$} {
				\If{ $\mathcal{S}[q_{adj}] = 1$} {
					$AddEdges(Slices[q_{cur}],Slices[q_{adj}])$ \label{ln:caller:add-edges}\\ }
			}
			$Slices[q_{cur}].reset()$ \\
		}
	}
	\Return{$\mathcal{G}$}
\end{algorithm}

The function $SelectQuaternion$ is used to select the next quaternion which has not been  seen before. In our implementation we use a simple scheme by drawing a random number $r$ and choosing the first $i > r$ so that $\mathcal{S}[i]=0$.

\section{Theoretical Properties of Our Approach}
In this section, we discuss correctness, completeness and computational complexity of our approach.

\subsection{Correctness}
First let us show that our algorithm is correct: i.e., if there is no collision-free path between two configurations in our approximation of the free space, then these configurations are also disconnected in the actual free space. 
\begin{proposition}[correctness]
\label{correctness}
Consider an object $\OO$ and a set of obstacles $\SS$. Consider two collision-free configurations of the object. If they are not path-connected in $\mathcal{G}(\afree_\varepsilon(\OO))$, then they are not path-connected in $\cfree(\OO)$.
\end{proposition}

\begin{proof}
Recall that the approximation of the free space is constructed as follows:

\[
\afree_\varepsilon(\OO) = \bigcup^{s}_{i = 1} \asfree_{U(\phi_i,\varepsilon)},
\]
where 
\begin{equation}
\label{eqn::free-slice-dfn}
\asfree_{U(\phi_i,\varepsilon)} = Dual(\ccol(\OO_{\varepsilon}^{\phi_i})) \times U(\phi_i,\varepsilon)
\end{equation}

Now, recall that by definition ${(Dual(\ccol(\OO_{\varepsilon}^{\phi_i})))}^c \subset \ccol(\OO_{\varepsilon}^{\phi_i})$~(\cite{edelsbrunner-skin}), and that we choose $\varepsilon$ and $U(\phi_i,\varepsilon)$ so that for any $\phi \in U(\phi_i,\varepsilon)$, $\ccol(\OO_{\varepsilon}^{\phi_i}) \subset \ccol(\OO^{\phi})$.
This implies that ${(Dual(\ccol(\OO_{\varepsilon}^{\phi_i})))}^c \subseteq \ccol(\OO^{\phi})$ for any $\phi\in U(\phi_i,\varepsilon)$, and conversely that $
\cfree(\OO^{\phi}) \subset {Dual(\ccol(\OO_{\varepsilon}^{\phi_i}))} $. Finally, since $\sfree_{U(\phi_i,\varepsilon)} = \bigcup_{\phi} \cfree(\OO^\phi) \times \{\phi\}$, we have:

\begin{equation}
Sl^{free}_{U(\phi_i,\varepsilon)} \subseteq aSl^{free}_{U(\phi_i,\varepsilon)},
\end{equation}

We now want to show that if there is no path between two vertices $v = (aC, U)$ and $v' = (aC', U')$ in $\mathcal{G}(\afree_{\varepsilon})$, then there is no path between connected components of $\afree_\varepsilon(\OO)$ corresponding to them. It is enough to show that if two vertices corresponding to adjacent slices are not connected by an edge, then they represent two components which are disconnected in $\sfree_{U}\cup\sfree_{U'}$.

Consider two adjacent slices $Sl_{U(\phi_i,\varepsilon)}$ and $Sl_{U(\phi_{j},\varepsilon)}$, and two path-connected components $C_1 \subset aSl^{free}_{U(\phi_i,\varepsilon)}$ and $C_2 \subset aSl^{free}_{U(\phi_{j},\varepsilon)}$. Let $aC_1$ and $aC_2$ be their respective representations as unions of balls.

Let $v_1$ and $v_2$ be the vertices of $\mathcal{G}(\afree_{\varepsilon}(\OO))$ corresponding to these components: $v_1 = (aC_1, U(\phi_i,\varepsilon))$ and $v_2 = (aC_2, U(\phi_{j},\varepsilon))$. By construction, these are adjacent slices, therefore $U(\phi_i,\varepsilon)\cap U(\phi_j,\varepsilon) \neq \emptyset$ and since there is no edge between $v_1$ and $v_2$, we get $aC_1\cap aC_2 = \emptyset$. But, by construction $C_1 \subseteq aC_1\times U(\phi_i,\varepsilon)$ and $C_2 \subseteq aC_2\times U(\phi_j,\varepsilon)$, therefore, we get:

\[
    C_1 \cap C_2\subseteq \left(aC_1\times U(\phi_i,\varepsilon)\right) \cap \left(aC_2\times U(\phi_j,\varepsilon)\right) = \emptyset
\]

And so $C_1$ and $C_2$ are disjoint in the union of the corresponding slices, which concludes the proof. 
\end{proof}

% =================================================================
%\begin{comment}

\subsection{$\delta-$Completeness}

Now, we show that if two configurations are not path-connected in $\cfree(\OO)$, we can construct an approximation of $\cfree(\OO)$ in which these configurations are either disconnected or connected by a narrow passage. 

\begin{proposition}[$\delta$-completeness]
\label{completeness}
Let $c_1, c_2$ be two configurations in $\cfree(\OO)$. If they are not path-connected in $\cfree(\OO)$, then for any $\delta > 0$ there exists $\varepsilon > 0$ such that the corresponding configurations are not path-connected in $\mathcal{G}(\afree_\varepsilon(\OO_{+\delta}))$, where the graph is produced according to the procedure outlined in Rem.~\ref{exact-graph}.
\end{proposition}

In proving this proposition we make used of the notion of the signed distance between two sets:

\[
\bar\dsign(\OO,\SS) =
\begin{cases}
	\min_{p\in O} d(p,\SS) & \textrm{if } \OO\cap\SS\neq\emptyset \\
	-\max_{p\in \OO\cap\SS} d(p,\partial\SS) & \textrm{otherwise}
\end{cases}
\]

Where $\partial\SS$ denotes the boundary of $\SS$. Note that $\bar\dsign(A,B)$ is not necessarily the same as $\bar\dsign(B,A)$ so we instead consider $\dsign(A,B) = \min(\bar\dsign(A,B),\bar\dsign(B,A))$.

\begin{proof}
%First, observe that if two configurations $c_1$ and $c_2$ are not path-connected in $\cfree(\OO)$, then for any $\delta > \varepsilon$ they are not $\delta-$connected in $\cfree(\OO_\varepsilon)$.

    Recall Rem.~\ref{exact-graph} that there is an edge between vertices $(aC_1,\phi_1)$, $(aC_2,\phi_2)$ only if $U(\phi_1,\varepsilon)$ overlaps with $U(\phi_2,\varepsilon)$ and $C_1$ overlaps with $C_2$ where $C_i = aC_i\cap {\cfree(\OO_\varepsilon^{\phi_i})}$ for $i = 1,2$ i.e.~the components of the actual free space of $\OO_\varepsilon^{\phi_i}$ ($i=1,2$) corresponding to the approximations $aC_1$ and $aC_2$.% This means that we can perform the analysis in terms of collisions in workspace, rather than looking at the configuration space.

%First, observe that if two configurations $c_1$ and $c_2$ are not path-connected in $\cfree(\mathcal{O})$, then for any $\delta > \varepsilon$ they are not $\delta-$connected in $\cfree_\varepsilon(\mathcal{O}(\mathcal{O})$ for a particular choice of $\Delta \phi$. Let $\Delta \phi$ be such that $\D(\Delta \phi) < \delta - \varepsilon$, where $\D(\Delta \phi)$ is the displacement function.

Recall now that we want to prove that for a pair of configurations $c_1,c_2$ which are not path-connected in $\cfree(\OO)$, then for any $\delta>0$ there exists some $\varepsilon > 0$ so that they are not path-connected in $\mathcal{G}(\afree_{\varepsilon}(\OO_{+\delta}))$ for.
Therefore, we start by noting that since $c_1$ and $c_2$ are not path-connected there exists a collision configuration $c$ in any path between them, which implies $\dsign(c(\OO),\SS) < 0$. Thus, for the same configuration $c$ we have $\dsign(c(\OO_{+\delta}),\SS) < -\delta$.

To see that this will result in path non-existence, we take an arbitrary $\varepsilon > 0$ and consider the collision space $\acol_{\varepsilon}(\OO_{+\delta})$. Further, we let $c = (p,\phi) \in \R^3\times SO(3)$, and for any $\phi_i$ such that $\phi\in U(\phi_i,\varepsilon)$, we define $c_i = (p,\phi_i)$. Finally, we restrict ourselves to the case where $\varepsilon < \delta$ and define $\delta' = \delta-\varepsilon$, then we get:

\begin{align*}
    \dsign(c_i(\OO_{+\delta'}),\SS)
    & \leq d(c(\OO_{+\delta'}), c_i(\OO_{+\delta'})) + \dsign(c(\OO_{+\delta'}),\SS) \\
    & \leq \varepsilon\left( 1 + \frac{\delta'}{\rad(\OO)}\right) - \delta'\\
\end{align*}
Where the term $\varepsilon\left( 1 + \frac{\delta'}{\rad(\OO)}\right)$ corresponds to the maximum displacement of the object offset $\OO_{+\delta}$ associated to a rotation of $\OO$ with maximum displacement $\varepsilon$.
Which implies that, as long as we choose $\varepsilon$ such that
\[
\varepsilon\left( 1 + \frac{\delta - \varepsilon}{\rad(\OO)}\right) - (\delta - \varepsilon) < 0
\]
then $c$ is in collision and the proof will be concluded. To see that this inequality has a solution with $0<\varepsilon<\delta$ we simply note that with $\varepsilon = 0$ it is reduced to $-\delta  < 0$. And the roots of the equation in terms of $\delta$ are given by:
{\small
\[
\varepsilon =
\frac{\rad(\OO)}{2}\left(2 + \frac{\delta}{\rad(\OO)}
\pm\sqrt{\left(2 + \frac{\delta}{\rad(\OO)}\right)^2 - \frac{4\delta}{\rad(\OO)}}\right)
\]
}
which are both positive. Finally call the smallest of the two roots $\tilde\varepsilon$ and the inequality holds as long as $\varepsilon \in (0,\min\{\tilde\varepsilon,\delta)\}$, concluding the proof.
\end{proof}

% =====================================================================
%\end{comment}

\subsection{Computational complexity} 
Let us now discuss the total computational complexity of the algorithm. Let $s$ be the number of slices, and let $n$ and $m$ be the number of balls in the object's and the obstacle's spherical representation, respectively. We have pre-computed the grids over $SO(3)$ corresponding to different dispersion values, and  therefore we are only interested in the complexity of the connectivity graph construction.  For each slice, we execute two computationally expensive procedures: we compute a weighted Voronoi diagram of the collision space, which allows us to extract the balls representation of the free space, and then for each slice we compute its intersections with adjacent slices. In practice, each orientation in $\mathcal{Q}$ has around $20$ adjacent orientation values, so each slice has around $20$ neighbours\footnote{This is the case for $SO(3)$, in the case of $SO(2)$ there are exactly $2$ neighbours.}.

In CGAL representation, the regular triangulation contains the corresponding weighted Voronoi diagram. Note that a weighted Voronoi diagram can be constructed by other means using for example the algorithm from~(\cite{aurenhammer}). The complexity of this step is $O(n^2m^2)$. The computation of the connected components of each slice is linear on the number of balls in the dual diagram, which makes the overall complexity of this step $O(n^2m^2)$. 

The complexity of finding the intersections between two connected components belonging to different slices 
 $O(b \log^3(b)+k)$ in the worst-case~\cite{zomorodian}, where $b$ is the number of balls in both connected components, and $k$ the output complexity, i.e., the number of pairwise intersections of the balls.

The complexity of the final stage of the algorithm --- computing connected components of the connectivity graph --- is linear on the number of vertices, and can be expressed as $O(s \ c)$, where $c$ is the average number of connected components per slice (a small number in practice).

%
%\begin{comment}
%In Sec.~\ref{our_approach} we said that the partition of $SO(2)$ is constructed in such a way that the distance $\Delta \phi$ between two consecutive orientation samples should be chosen such that the maximal displacement does not exceed the chosen $\varepsilon$. Note that $\varepsilon$ should be less than the radius $r$ of the smallest ball in the object's representation in order to preserve the shape of the object. On the other hand, Fig.~\ref{plot} shows that the number $s$ grows significantly when we decrease the $\varepsilon$.
%
%
%\begin{figure}[htb!]   
%\center{\includegraphics[width=0.4\textwidth]{figures/plot.png}}
%  \caption{This plot shows how the number of slices (Y-axis) depends on the $\varepsilon$ (X-axis) given an object of diameter 5.}
%  \label{plot} 
%  %\vspace*{-.4cm}
%\end{figure}
%\end{comment}

%Note that, in particular, when we prove that an object is not caged using an $\varepsilon$-core, we actually prove that there is an escaping path with clearance larger than $\varepsilon$.

\section{Examples and Experiments}
\label{experiments}

\begin{figure*}[htb!]
\centering
\subfloat{\includegraphics[width=3cm]{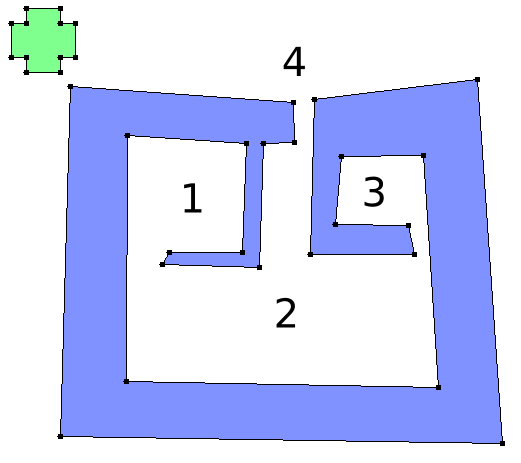}}
\subfloat{\includegraphics[width=3.3cm]{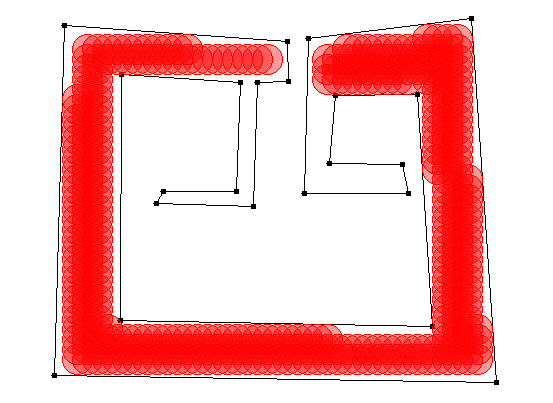}} 
\subfloat{\includegraphics[width=3.3cm]{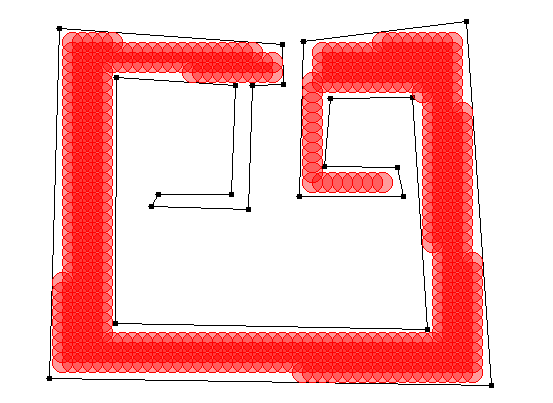}}
\subfloat{\includegraphics[width=3.3cm]{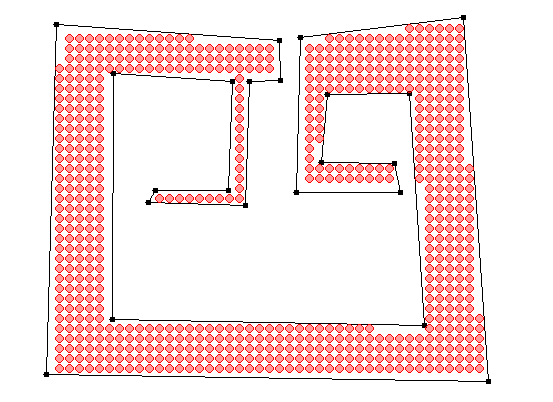}}
\caption{\label{examples2D} In the left we present a set of 2D obstacles in blue, and an object in green. The numbers represent the connected components of the free space. In red are three approximations of the obstacles by sets of balls of radius 15, 10, and 4, respectively. Note that the smaller the radius the more features from the original configuration space are preserved.
Note that the first approximation significantly simplifies the shape, and has only one narrow passage; the second approximation preserves the shape better and has two narrow passages; the third approximation preserves all the important shape features of the obstacles.}
\end{figure*}

In this section, we test our algorithm on 2D and 3D objects. In the first experiment, we investigate how the quality of the spherical approximation of the workspace  affects the output of our algorithm. In the next experiment, we test our algorithm on a set of 3D objects representing different classes of geometric shapes. In our final experiment we investigate how the performance of our parallel implementation changes with respect to different numbers of threads.

\subsection{Object and obstacle approximation}

A key requirement of our method is the approximation of both objects and obstacles as unions of balls. This presents a challenge especially when dealing with 3D objects. For 2D experiments we were able to obtain good results using a simple grid for the centers of the balls, and choosing the radius of the balls at various levels. This resulted in the approximations of the obstacle seen in Fig.~\ref{examples2D} and the results in Tab.~\ref{time} which will be analyzed in the next section. Simply put this example shows that using this approach, the smaller the radius of the balls used in the approximation the higher detail we are able to capture. However this method can be detrimental to approximate the object, because we use an $\varepsilon$-core, where $\varepsilon$ depends on the dispersion of the grid over $SO(n)$ and the radius of the object and may therefore exceed the necessary ball radius to approximate all the required details.

A more suitable approach to approximate the object is to compute the medial axis of the object using for example the method from~\citep{dey-medialaxis-pointcloud} which is able to construct it from a pointcloud of the object. Alternatively one can use the method from~\citep{dey-medialaxis-mesh} which is able to retrieve a skelleton of the object (i.e. a 1-dimensional subset of the medial axis). In our experiments we used a water-tight mesh of the drill (see Fig~\ref{examples3Dreal}) to extract its skelleton axis and obtained an approximation of the drill by sampling ball centers in the skelleton with the largest radius that did not contain any point in the mesh.

Finally, for the mug and the bugtrap (also Fig.~\ref{examples3Dreal}) we traced a slice of the object and manually fit the medial axis of the slice, which we used to create a surface of revolution where the ball centers were placed (bugtrap and body of the mug).

\subsection{2D scenario: Different Approximations of the Workspace}

In this experiment, we consider how the accuracy of a workspace spherical approximation affects the output and execution time of the algorithm, see Fig.~\ref{examples2D}. This experiment was performed on a single thread of Intel Core i7 processor.

\begin{table}[htb!]
\centering
\begin{tabular}{|@{~} c @{~}|@{~} r r @{~}|@{~} r r @{~}|@{~} r r @{~}|}
\hline
  $\varepsilon$ & \multicolumn{2}{c|@{~}}{$R = 15$} & \multicolumn{2}{c|@{~}}{$R = 10$} & \multicolumn{2}{c|}{$R = 4$}   \\
\hline
   $0.30 \cdot r$ & 2 c.& 741 ms & 3 c.& 1287 ms & 4 c.& 1354 ms \\
   $0.33 \cdot r$ & 2 c.& 688 ms & 3 c.& 1208 ms & 4 c.& 1363 ms \\
   $0.37 \cdot r$ & 2 c.& 647 ms & 3 c.& 1079 ms & 4 c.& 1287 ms \\
   $0.40 \cdot r$ & 2 c.& 571 ms & 3 c.&  986 ms & 3 c.& 1156 ms \\
   $0.43 \cdot r$ & 2 c.& 554 ms & 3 c.&  950 ms & 3 c.& 1203 ms \\
\hline
\end{tabular}
\caption{\label{time}%Results from 2D experiments in Fig.~\ref{examples2D} with 5 values of $\varepsilon$ and the presented 3 workspace approximations. In each case, we report the number of path-connected components (c) and the runtime in miliseconds. Note that the coarser the approximation is, the fewer components are detected, and the quicker the algorithm terminates.}
We report the number of path-connected components we found and the computation time for each case. When we were using our first approximation of the workspace, we were able to distinguish only between components 4 and 2 (see Fig.~\ref{examples2D}), and therefore prove path non-existence between them. For a more accurate approximation, we were also able to detect component 3. Finally, the third approximation of the workspace allows us to prove path non-existence between every pair of the four components.}
\end{table}

For our experiments, we generate a workspace and an object as  polygons, and approximate them with unions of balls of equal radii lying strictly inside the polygons. Note that the choice of the radius is important: when it is small, we get more balls, which increases the computation time of our algorithm; on the other hand, when the radius is too large, we lose some important information about the shape of the obstacles, because narrow parts cannot be approximated by large balls, see Fig.~\ref{examples2D}.

We consider a simple object whose approximation consists of 5 balls. We run our algorithm for all the 3 approximations of the workspace, and take 5 different values of $\varepsilon$, see Tab.~\ref{time}. We can observe that as we increase the $\varepsilon$ the computation time decreases. This happens because we are using fewer slices. However, we can also observe that when the $\varepsilon$ is too large, our approximation of the collision space becomes too small, and we are not able to find one connected component (see the last column of Tab.~\ref{time}).

 In our opinion, the accuracy of a workspace approximation should depend on the application scenario as well as on the amount of avalibale computational resources.

\subsection{3D scenario: Different Types of 3D Caging}

\newcolumntype{C}{>{\centering\arraybackslash} m{6cm} }
\newcolumntype{D}{>{\centering\arraybackslash} m{4cm} }

\begin{figure*}[htb!]
\centering
\begin{tabular}{DDD}
\includegraphics[height=2.4cm]{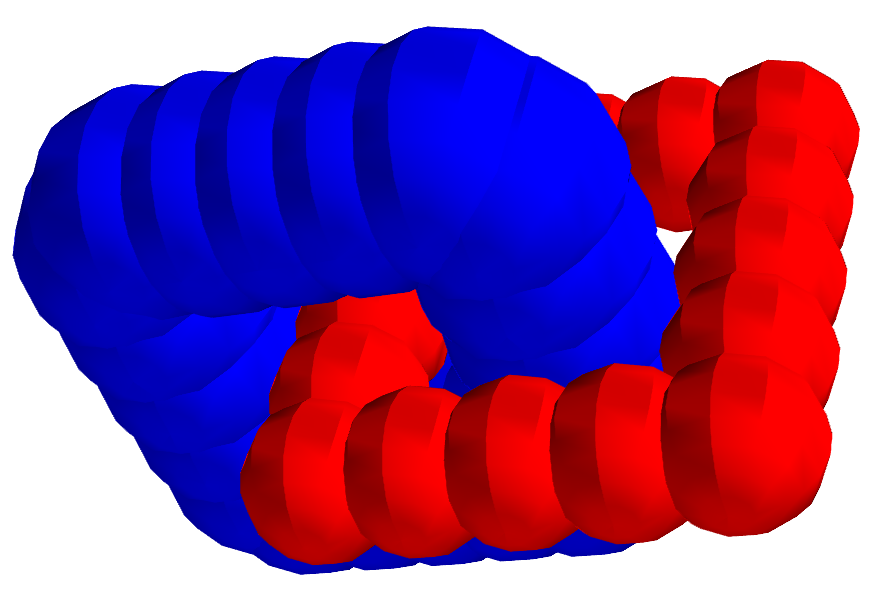} &
\includegraphics[height=3.2cm]{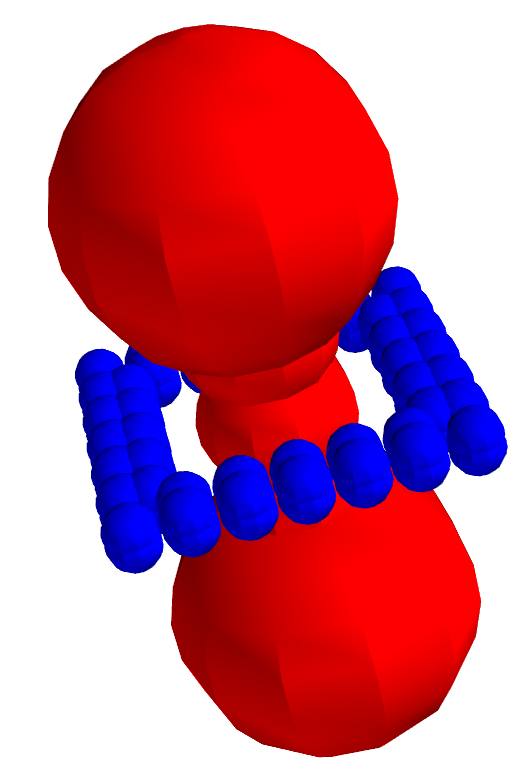} &
\includegraphics[height=3.2cm]{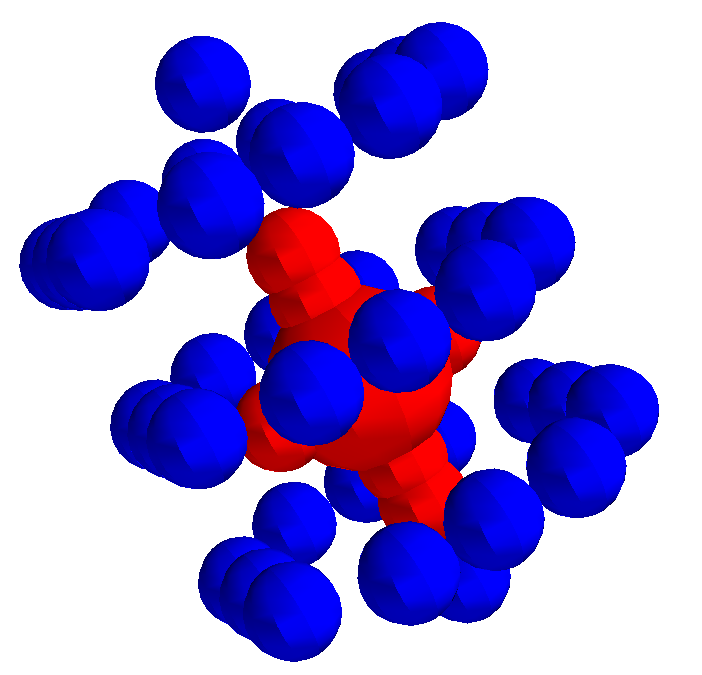}
\end{tabular}
%\subfloat{\includegraphics[width=4cm]{figures/gravity.png}}
\caption{\label{examples3D} Different 3D caging scenarios. From left to right: \emph{linking-based caging}~(\cite{pokorny,stork2013b}), \emph{narrow part-based caging}~(\cite{varava_2,icra_caging}), \emph{surrounding-based caging}.% , and \emph{caging with gravity}~\cite{mahler}.
In linking-based and narrow part-based caging the obstacles form a loop (not necessarily closed) around a handle or narrow part of the object. In surrounding-based the object is surrounded by obstacles so as not to escape. 
}
\end{figure*}

As we have mentioned in Sec.~\ref{related-work}, a number of approaches towards 3D caging is based on identifying shape features of objects and deriving sufficient caging conditions based on that. By contrast, our method is not restricted to a specific class of shapes, and the aim of this section is to consider examples of objects of different types and run our algorithm on them. The examples are depicted on Fig.~\ref{examples3D}, and Tab.~\ref{results} reports execution time for different resolutions of the $SO(3)$ grid.

\begin{table*}[htb!]
\centering
\begin{tabular}{| r | r r | r r | r r |}
  \hline
  \multicolumn{1}{|c|}{\#slices}& \multicolumn{2}{c|}{rings (320)} & \multicolumn{2}{|c|}{narrow part (480)}& \multicolumn{2}{|c|}{surrounding (266)} \\
  \hline
        576 &   5 s & $\delta = 3.153$ &   8 s & $\delta = 5.041$ &   9 s & $\delta = 1.390$ \\
      4 608 &  45 s & $\delta = 0.958$ &  78 s & $\delta = 1.532$ &  95 s & $\delta = 0.422$ \\
     36 864 & 485 s & $\delta = 0.000$ & 674 s & $\delta = 0.000$ & 875 s & $\delta = 0.000$ \\
%        576 &   5 s & $\delta = 2.4$ &   6 s & $\delta = 5.9$ &    10 s & $\delta = 1.2$\\
%      4 608 &  49 s & $\delta = 0.7$ &  70 s & $\delta = 2.4$ &   108 s & $\delta = 0.4$\\
%     36 864 & 540 s & $\delta = 0.0$ & 785 s & $\delta = 0.0$ & 1 073 s & $\delta = 0.0$\\
  \hline
\end{tabular}
\caption{\label{results}
Results from running 3D experiments on the objects shown on Fig.~\ref{examples3D} using 3 different resolutions for the $SO(3)$ grid with the non-parallelized algorithm. The number of balls used to approximate the collision space of each model is indicated in parenthesis next to the model name. We report the computational time and the value of clearance $\delta$ used to test whether there is a narrow passage of that clearance for each case. The varying clearance corresponds to using the same value for epsilon on each grid.}
\end{table*}

In all cases (Tab.~\ref{results}), the object was shown to be caged  when using the third grid. In the case of the narrow part example the algorithm was not able to discern a cage using the first grid. In the rings example it was not able to find the cage with either the second or third grids.

Results that are based on low-resolution grids can in certain cases be inaccurate for two reasons. First, the they have higher dispersion value, which implies that we need to use a larger value of $\varepsilon$. This results in a smaller $\varepsilon-$core of the object which can escape from a cage even if the entire object cannot. This effect can be mitigated by considering a sufficient clearance $\delta$. Thus, depending on the desired accuracy of the resulting approximation, different numbers of slices can be used. In our current implementation, we precompute grids on $SO(3)$ using the algorithm by~\cite{yershova}, and in our experiments we consider 3 different predefined resolutions. Using a different $SO(3)$-discretization strategy and given a concrete clearance value $\delta$, one can potentially construct a grid on $SO(3)$ with the necessary resolution based on $\delta$.

Another reason for inaccuracies in the low-resolution grid is the fact that the distance between neighboring orientations is so large that adjacent slices might have drastic differences in the topology of their free space. This may lead to adding edges between their connected components that would not be connected if we considered intermediate orientations as well in the case of a more fine grid.

As we see, our algorithm can be applied to different classes of shapes. It certainly cannot replace previous shape-based algorithms, as in some applications global geometric features of objects can be known a priori, in which case it might be faster and easier to make use of the available information. However, in those applications where one deals with a wide class of objects without any knowledge about shape features they might have, our algorithm provides a reasonable alternative to shape-based methods.

\subsection{Parallel implementation}

Our parallel implementation allows one to run experiments on more complex models of real world objects. In this section, we consider a few objects representing different geometric features, see Fig.~\ref{examples3Dreal}. These examples are inspired by manipulation and motion planning applications. The first two models, a mug and a drill, are caged by a Schunk dexterous hand. The mug and the drill illustrate linking-based and narrow part-based caging types, respectively. The third example is a bugtrap and a small cylindrical object. The bugtrap environment is a well-known path planning benchmark created by Parasol MP Group, CS Dept, Texas A \& M University. The bugtrap does not cage the object, but contains a narrow passage. This example corresponds to the ``surrounding'' type of restricting the object's mobility.

First, we observe how the performance of our algorithm changes with respect to the number of slices we consider. Table~\ref{results-real} reports the results. The respective objects are depicted on Fig.~\ref{examples3Dreal}. We report the computational time needed for different resolutions of the orientation grid, as well as the $\delta$ value passed to the algorithm. The different $\delta$ values correspond to mantaining a constant value of $\varepsilon$ across the different grids. In all of the examples, the algorithm was able to detect compact connected components, thus indicating that the objects are caged with clearance $\delta$. In the case of a bugtrap,  we were able to detect a narrow passage whose width is at most $0.2$. However, when running the same example with clearance parameter $\delta = 0$, the algorithm returned a space approximation with a single connected component, thus indicating that there are no caging configurations. Thus, while the object is not caged by the bugtrap, its free space contains a narrow passage as expected.

\begin{figure*}[htb!]
\centering
\begin{tabular}{DDC}
{\includegraphics[height=3cm]{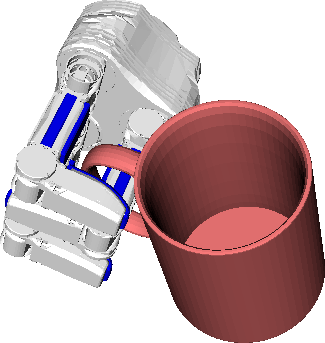}} & 
{\includegraphics[height=3cm]{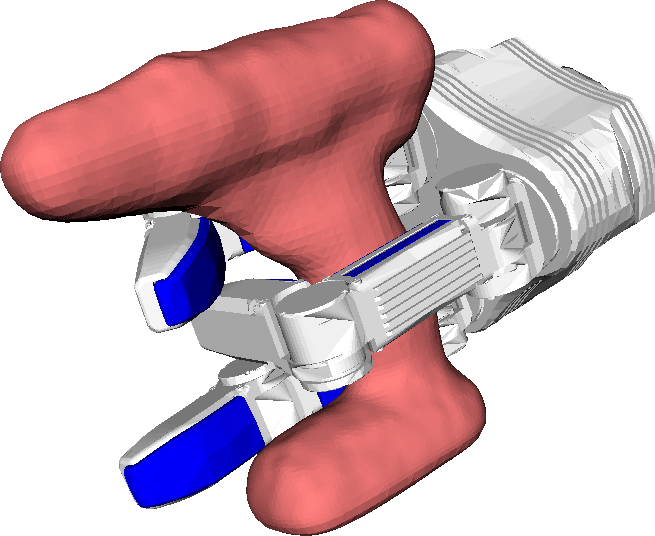}} &
{\includegraphics[height=3cm]{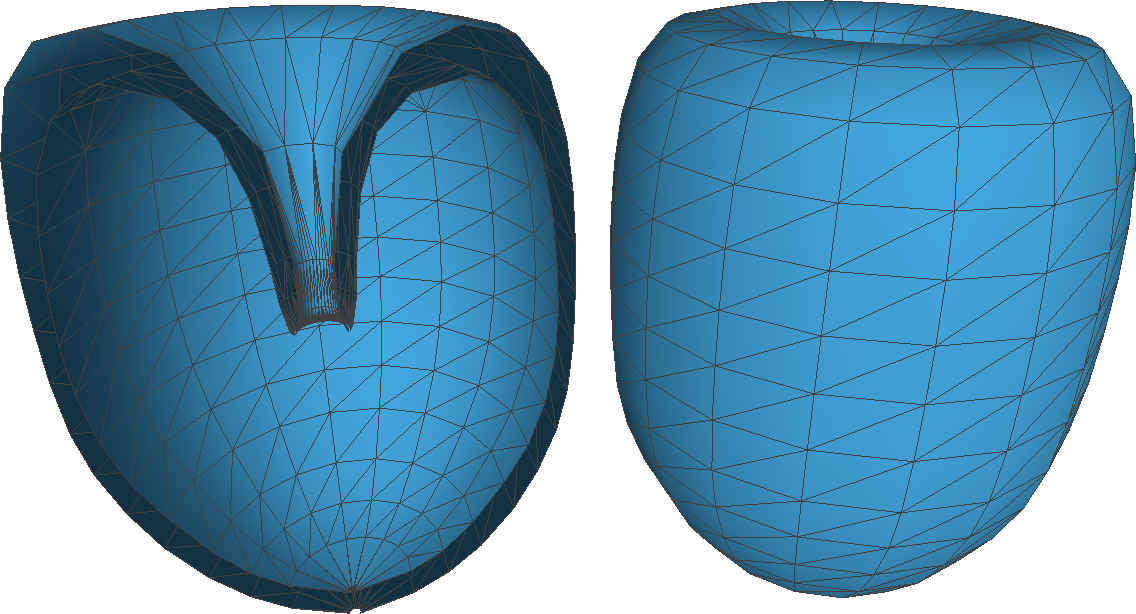}} \\
{\includegraphics[height=2.2cm]{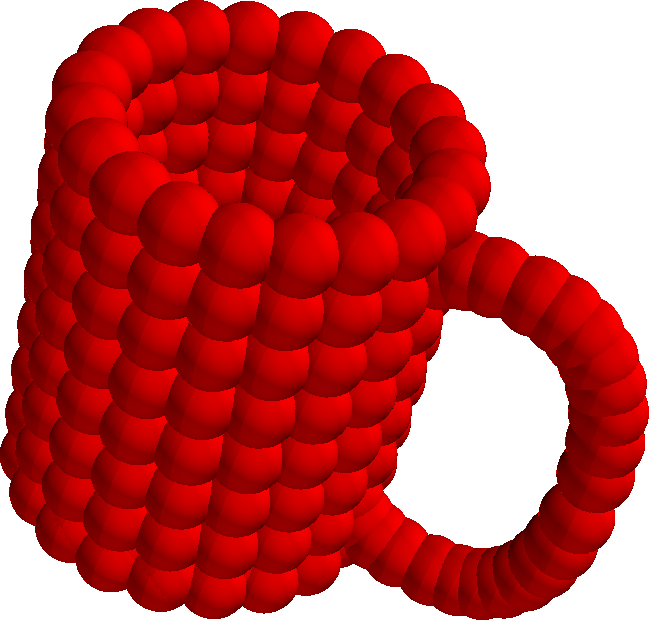}} &
{\includegraphics[height=2.4cm]{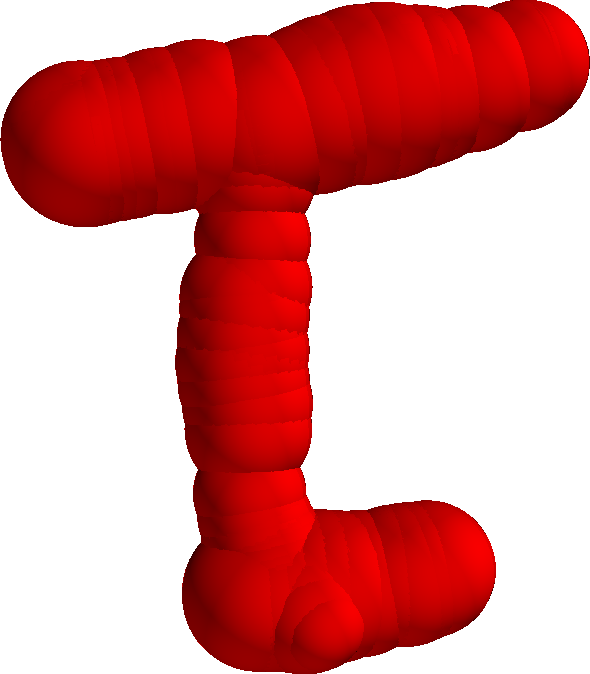}} &
{\includegraphics[height=.3cm]{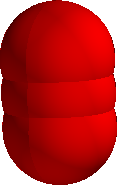}}  \\
{\includegraphics[height=2.5cm]{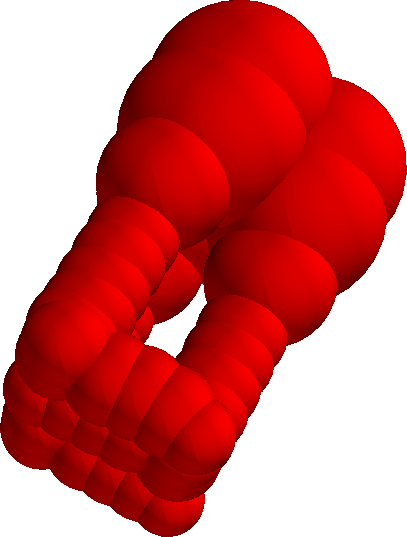}} &
{\includegraphics[width=3cm]{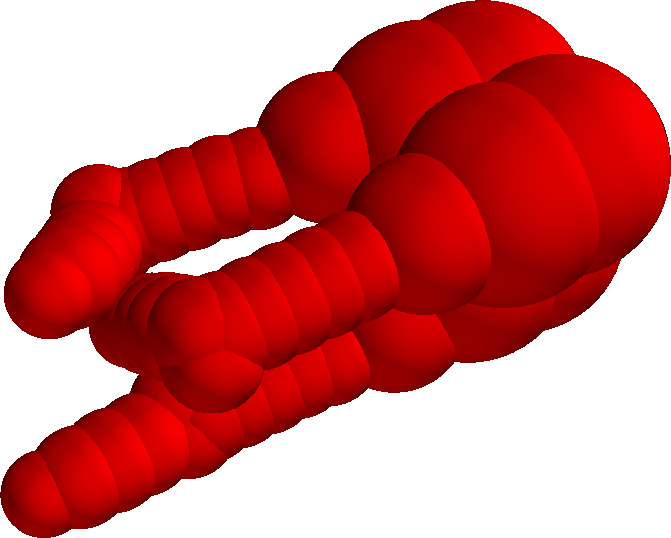}} &
{\includegraphics[height=3cm]{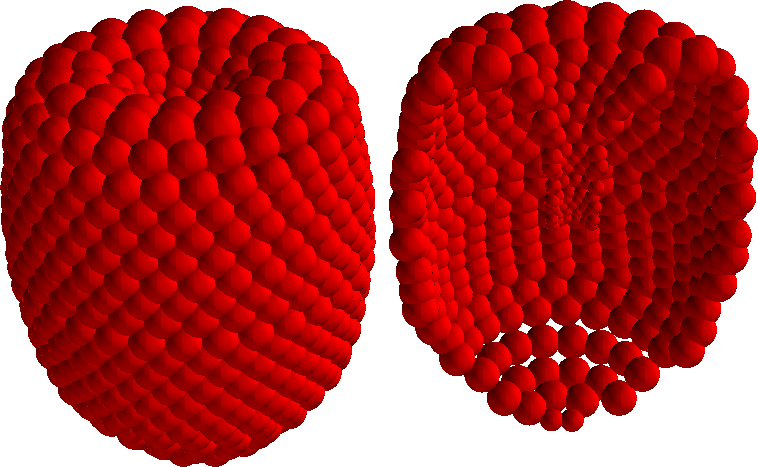}}
\end{tabular}
\caption{\label{examples3Dreal}  Caging real-world objects with a 3 finger Schunk hand. First row, from left to right: a mug (linking-based caging), its representations as a union of balls, a drill (narrow part-based caging), a bugtrap (caging by surrounding). Second row: the respective spherical representations of the objects. Third row: a spherical representation of the Schunk hand in the caging configurations used above; an object passing through the bugtrap.}
\end{figure*}

\begin{figure}[htb]
\centering
\includegraphics[width=.8\columnwidth]{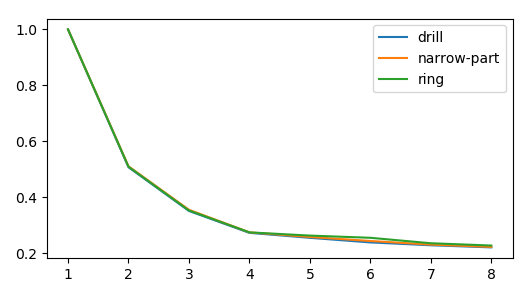}
\caption{\label{fig:threads}  Relative timing for calculating the configuration space of objects with different numbers of threads. The vertical axis measures time in fraction of time taken by computing with a single thread. The horizontal axis indicates the number of threads. The data corresponds to that presented in Tab.~\ref{results-parallel}.}
\end{figure}

\begin{table*}[htb!]
\centering
\begin{tabular}{| r | r r | r r | r r |}
  \hline
\multicolumn{1}{|c|}{\#slices}& \multicolumn{2}{c|}{ mug (9198 balls) }& \multicolumn{2}{c|}{drill (2646 balls)} & \multicolumn{2}{c|}{bugtrap ( 2079 balls) } \\
\hline
	   576	&    36 s & $\delta = 0.041$ &  12 s & $\delta = 0.044$ &   169 s & $\delta = 0.756$ \\
	 4 608	&   230 s & $\delta = 0.012$ & 110 s & $\delta = 0.013$ & 1 095 s & $\delta = 0.369$ \\
	36 864	& 1 765 s & $\delta = 0.000$ & 784 s & $\delta = 0.000$ & 8 263 s & $\delta = 0.200$ \\
\hline
\end{tabular}
\caption{\label{results-real}
Results from running 3D experiments with objects shown on in Fig.~\ref{examples3Dreal} using 3 different resolutions for the $SO(3)$ grid. The number of balls used to approximate the collision space of each model is indicated in parenthesis next to the model name.}
\end{table*}

\begin{table}[htb!]
\centering
\begin{tabular}{| c | c | c | c |}
  \hline			
  \#threads & ring & narrow part & drill \\
  \hline
	1 & 95 s & 79 s & 769 s \\
	2 & 48 s & 40 s & 390 s \\
	3 & 33 s & 28 s & 270 s \\
	4 & 26 s & 22 s & 210 s \\
	5 & 25 s & 20 s & 196 s \\
	6 & 24 s & 19 s & 184 s \\
	7 & 22 s & 18 s & 176 s \\
	8 & 22 s & 18 s & 170 s \\
% source (times in milliseconds) the examples were run on the second grid (4608 vertices):
%
% NCores drill molecule mug narrow-part ring
% 1  768631 181861 8881569 78739 94610
% 2  390065 91839 4564100 40223 48188
% 3  269993 63215 3219878 27988 33400
% 4  209831 48908 2607975 21644 25994
% 5  196171 46503 2596021 20329 24904
% 6  183588 44093 2706905 19210 24159
% 7  175748 42401 2849265 18189 22316
% 8  169947 41179 3360344 17540 21538
%
\hline  
\end{tabular}
\caption{\label{results-parallel}
In this table we show the results from running the experiments on the ring, and narrow part toy examples, as well as on the drill objects, with one to eight threads. These examples stem from using our parallel algorithm (Alg.~\ref{alg:parallel-caller}) with the objects and obstacles depicted in Fig.~\ref{examples3D} and Fig.~\ref{examples3Dreal}, and a grid over $SO(3)$ comprising 4608 nodes.}
\label{table:threads-results}
\end{table}

Now, we analyze how the performance of the parallelized version of our algorithm changes with respect to the number of threads we use. Table~\ref{table:threads-results} presents the results of this experiment. We run our algorithm using from one to eight threads on three different object models and report the respective computational time. In all examples we used 4608 slices (second grid) and clearance $\delta=0$. We observe a drastic increase of performance up to four threads, and a slight increase when using four to eight threads (see Fig.~\ref{fig:threads}).

%+++++++++++++++++++++++++++++++++++++++++++++++++++++++++++++++++++++++++
\section{Discussion and Future Work} 

In this paper, we provide a computationally feasible and provably-correct algorithm for 3D and 6D configuration space approximation and use it to prove caging and path non-existence. We analyze theoretical properties of the algorithm such as correctness and complexity. We provide a parallel implementation of the algorithm, which allows us to run it on complex object models. In the future, we see several main directions of research that would build upon these results.

\subsection{Molecular Cages}

In organic chemistry, the notion of caging was introduced independently of the concept of~\cite{Mitra2013}. Big molecules or their ensembles can form hollow structures with a cavity large enough to envelope a smaller molecule. Simple organic cages are relatively small in size and therefore can be used for selective encapsulation of tiny molecules based on their shape~\cite{Mitra2013}.  In particular, this property can be used for storage and delivery of small drug molecules in living systems~\cite{rother}. By designing the caging molecules one is able to control the formation and breakdown of a molecular cage, in such a way as to remain intact until it reaches a specific target where it will disassemble releasing the drug molecule. 

An example on Fig.~\ref{chemistry} shows that in principle, our algorithm can be applied to molecular models, assuming they can be treated as rigid bodies. In this example, we consider two molecules to see whether our algorithm can predict their ability to form a cage. Atomic coordinates and van der Waals radii were retrieved from the Cambridge Crystallographic Data Centre (CCDC) database. The depicted molecules are mesitylene (left, CCDC 907758) and CC1 (right, CCDC 707056) (\cite{Mitra2013}). Our algorithm reported that this pair forms a cage, as experimentally determined in~\cite{Mitra2013}. In our future work, we aim to use our algorithm to test a set of potential molecular carriers and ligands to find those pairs which are likely to form cages. This prediction can later be used as a hypothesis for further laboratory experiments.

\subsection{Integration with Path Planners}
Another possible direction of future work is integration of our approach with sampling-based path planning algorithms. Since the problems of path planning and path non-existence verification are dual to each other, we plan to design a framework where they will be addressed simultaneously: a path planner will attempt to construct a collision-free path as usual, while configuration space approximation will be constructed independently. If there is no path, our configuration space approximation can be used to rigorously demonstrate this instead of relying on heuristical stopping criteria for the planner. Furthermore, we can leverage our approximation to guide sampling, which can be particularly beneficial in the presence of narrow passages. Namely, having an approximation of the narrow regions of the free space, the planner can sample configurations from it.

 \subsection{Energy-Bounded Caging}
 \cite{mahler, mahler_2} defined the concept of energy-bounded caging, when obstacles and external forces (such as gravity, or forces directly applied to the object) complement each other in restricting the object's mobility. The authors proposed an approach towards synthesizing energy-bounded cages in 2D.
 
 To directly extend their method to 3D workspaces one would need to represent the configuration space as a 6D simplicial complex, which is expensive in terms of both required memory and execution time ~(\cite{mahler}). Apart from that, the computation time in their case is dominated by sampling the collision space, and the authors suggest that in the 3D case the necessary number of samples might be even higher.
 
 Analogously to the works by ~\cite{mahler, mahler_2}, we can model external forces as potential functions and assign potential values to each ball in our approximation of the free space. In each slice, we consider balls with high potential values as ''forbidden" parts of the free space, and compute connected components of the rest. In the future, we plan to investigate this direction.

\begin{figure}[htb!]
\centering
\includegraphics[width=0.6\linewidth]{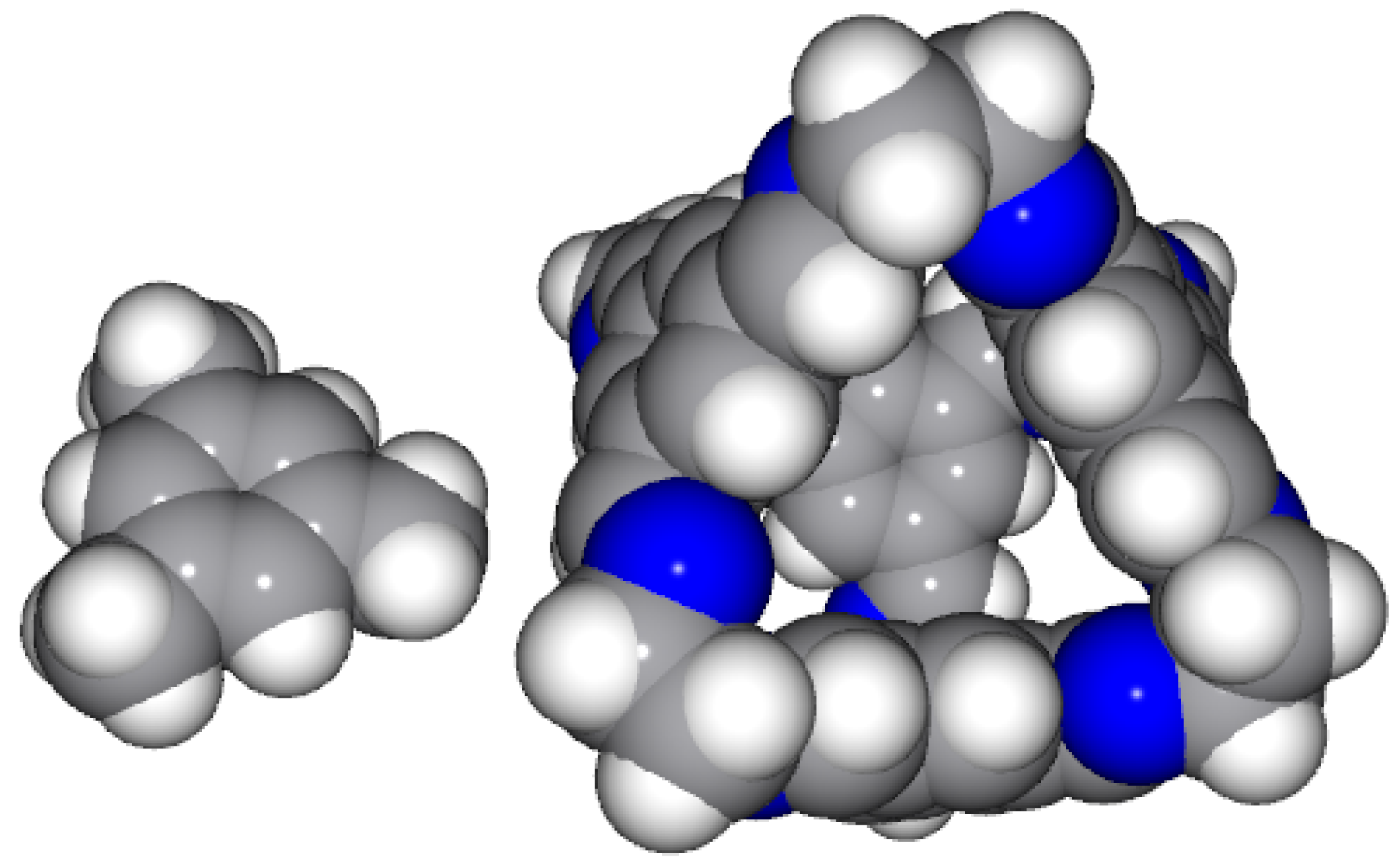}
\caption{\label{chemistry}The small molecule (a guest) can be caged by the big molecule (the host).}
\end{figure}

\begin{comment}
\section{\label{conclusion} Conclusion}
In this paper, we provide a computationally feasible and provably-correct algorithm for 3D and 6D configuration space approximation and use it to prove caging and path non-existence. In the future, we will look into how we can make our algorithm faster and applicable to real-time scenarios. In particular, we will work on a parallel implementation, and investigate how the number of slice can be reduced. We are also interested in integrating our narrow passage detection and path non-existence techniques with probabilistic motion planners. Finally, we are going to use our approach to find molecular cages.
\end{comment}

\section{Acknowledgements}
This work has been supported by the Knut and Alice
Wallenberg Foundation and Swedish Research Council. The authors are grateful to D. Devaurs, L. Kavraki, and O. Kravchenko for their insights into molecular caging.

\newpage
\appendix
\section{Appendix}

Recall the formulation of Proposition~\ref{prop:displacement-norm}.

\begin{quote}
	Given two unit quaternions $p,q$, the following equation holds:
    \begin{equation*}
    \D(R_{p\bar q}) = 2\sin(\rho(p,q)) \rad(\OO)
    \end{equation*}
\end{quote}

\begin{proof}
	For notational simplicity we divide both sides of the equation by $\rad(\OO)$ and define $\bar\D(R_{q}) = \frac{\D(R_q)}{\rad(\OO)}$. Which reduces the problem to proving:
    \[
	    \bar\D(R_{p\bar q}) = 2\sin(\rho(p,q))
    \]

    We proceed by reducing both sides of the equation to the same formula, starting with the left-hand-side. To do this we once again point out that we can identify quaternions with vectors in $\R^4$, and that a unit quaternion $q = \cos(\frac{\theta_q}{2}) + \sin(\frac{\theta_q}{2})(q_x i + q_y j + q_z k)$ is associated to a 3D rotation of an angle of $\theta_q$ around the axis $(q_x,q_y,q_z)$ which we will denote $w_q$.
    \begin{align*}%
	    &\bar\D(R_{p\bar q})%
	    = 2|\Im p\bar q| \\%
	    &= 2\| \cos(\frac{\theta_q}{2}) \sin(\frac{\theta_p}{2}) w_p - \cos(\frac{\theta_p}{2}) \sin(\frac{\theta_q}{2}) w_q \\%
	    &\quad- \sin(\frac{\theta_p}{2}) \sin(\frac{\theta_q}{2}) w_p \times w_q \| \\%
	      &= 2\sqrt{\| \cos(\frac{\theta_q}{2}) \sin(\frac{\theta_p}{2}) w_p - \cos(\frac{\theta_p}{2}) \sin(\frac{\theta_q}{2}) w_q \|^2}\\%
	      &\quad\overline{+ \|\sin(\frac{\theta_p}{2}) \sin(\frac{\theta_q}{2}) w_p \times w_q \|^2}%
    \end{align*}%
    Where the last equality is due to the fact that $w_p\times w_q$ is perpendicular to both $w_p$ and $w_q$, and is therefore a consequence of the Pythagorean theorem. Now recall that $\|w_p\times w_q\| = \sin(\omega_{p,q})$ where $\omega_{p,q}$ is the angle between $w_p,w_q$ and that $\langle w_p,w_q\rangle = \cos(\omega_{p,q})$, since $\|w_p\|=\|w_q\|=1$. Recall also that $\sin^2(\theta) = 1 - \cos^2(\theta)$, whence we obtain

    \begin{align*}
	    &\bar\D(R_{p\bar q})\\
	    &= 2\sqrt{\| \cos(\frac{\theta_q}{2}) \sin(\frac{\theta_p}{2}) w_p -
               \cos(\frac{\theta_p}{2}) \sin(\frac{\theta_q}{2}) w_q \|^2}\\
	       &\quad\overline{+ \sin^2(\frac{\theta_p}{2}) \sin^2(\frac{\theta_q}{2}) (1 - \langle w_p,w_q \rangle^2)}
    \end{align*}

    Furthermore let $\tilde w_p = \frac{w_q - \langle w_p,w_q\rangle w_p}{\|w_q - \langle w_p,w_q\rangle w_p\|}$ be the component of $w_q$ which is perpendicular to $w_p$, then we can rewrite $\cos(\frac{\theta_p}{2}) \sin(\frac{\theta_q}{2}) w_q$ as

    \begin{align*}
	    & \cos(\frac{\theta_p}{2}) \sin(\frac{\theta_q}{2}) w_q \\ & = \cos(\frac{\theta_p}{2}) \sin(\frac{\theta_q}{2}) (\langle w_p,w_q \rangle w_p + \langle \tilde w_p , w_q\rangle \tilde w_p )
    \end{align*}

    Substituting this into the formula, and using the pythagorean theorem to separate he $w_p$ and $\tilde w_p$ components, we can proceed with 

    \begin{align*}
	    & \bar\D(R_{p\bar q}) \\
	    &= 2\sqrt{\| (\cos(\!\frac{\theta_q}{2}\!)\! \sin(\!\frac{\theta_p}{2}\!)\!\! -\!\!
    \cos(\!\frac{\theta_p}{2}\!)\! \sin(\!\frac{\theta_q}{2}\!) \langle w_p,w_q\rangle) w_p\|^2}\\
              & \quad\overline{+\|\cos(\frac{\theta_p}{2}) \sin(\frac{\theta_q}{2}) \langle w_q,\tilde w_p\rangle \tilde w_p\|^2}\\
              & \quad\overline{+ \sin^2(\frac{\theta_p}{2}) \sin^2(\frac{\theta_q}{2}) (1 - \langle w_p,w_q \rangle^2)}
    \end{align*}

    Now we note that $\|w_q\| ,\|w_p\| , \|\tilde w_p\| = 1$, and therefore $\| \langle w_p ,w_q\rangle w_p + \langle w_q , \tilde w_p \rangle \tilde w_p \| ^ 2 = 1$ which by the Pythagorean theorem gives $\langle w_q , \tilde w_p\rangle^2 = 1 - \langle w_p,w_q\rangle^2$, which we can substitute once again.

    \begin{align*}
	    & \bar\D(R_{p\bar q}) \\
	    & = 2\sqrt{{( \cos(\frac{\theta_q}{2}) \sin(\frac{\theta_p}{2}) -
            \cos(\frac{\theta_p}{2}) \sin(\frac{\theta_q}{2}) \langle w_p,w_q\rangle )}^2}\\
               & \quad\overline{+\cos^2(\frac{\theta_p}{2}) \sin^2(\frac{\theta_q}{2}) (1 - \langle w_p,w_q \rangle^2)}\\
              & \quad\overline{+ \sin^2(\frac{\theta_p}{2}) \sin^2(\frac{\theta_q}{2}) (1 - \langle w_p,w_q \rangle^2)}\\
    \end{align*}

    Now we want to deal only with a combination of tangents, therefore we divide the term inside the square root by $\cos^2(\frac{\theta_p}{2})\cos^2(\frac{\theta_q}{2})$ yielding:

    \begin{align*}
	    & \bar\D(R_{p\bar q}) = 2|\cos(\frac{\theta_p}{2})\cos(\frac{\theta_q}{2})|\\ & \quad \sqrt{{( \tan(\frac{\theta_p}{2}) -
            \tan(\frac{\theta_q}{2}) \langle w_p,w_q\rangle )}^2}\\
               & \quad\overline{+\tan^2(\frac{\theta_q}{2}) (1 - \langle w_p,w_q \rangle^2)}\\
              & \quad\overline{+ \tan^2(\frac{\theta_p}{2}) \tan^2(\frac{\theta_q}{2}) (1 - \langle w_p,w_q \rangle^2)} \\
    \end{align*}

    Now, expanding the sqares and multiplying into all the terms under the squareroot sign, as well as eliminating terms that cancel out, results in:
    \begin{align*}
	    & \bar\D(R_{p\bar q}) = 2|\cos(\frac{\theta_p}{2})\cos(\frac{\theta_q}{2})|\\
	    &\quad \sqrt{%
            \tan^2(\frac{\theta_p}{2}) - 2\tan(\frac{\theta_p}{2}) \tan(\frac{\theta_q}{2}) \langle w_p,w_q\rangle}\\
	    & \quad\overline{+\tan^2(\frac{\theta_q}{2}) + \tan^2(\frac{\theta_p}{2}) \tan^2(\frac{\theta_q}{2}) }\\
	    &\quad\overline{-\tan^2(\frac{\theta_p}{2}) \tan^2(\frac{\theta_q}{2}) \langle w_p,w_q \rangle^2}\\
    \end{align*}

By introducing an extra $1-1$ into the square root, we can use these terms to complete products in order to simplify the equation.

    \begin{align*}
	    & \bar\D(R_{p\bar q})= 2|\cos(\frac{\theta_p}{2})\cos(\frac{\theta_q}{2})| \\
            %& = 2|\cos(\frac{\theta_p}{2})\cos(\frac{\theta_q}{2})|\sqrt{1 - 1 +
	    %\tan^2(\frac{\theta_p}{2}) } \\
    %& \quad\overline{- 2\tan(\frac{\theta_p}{2}) \tan(\frac{\theta_q}{2}) \langle w_p,w_q\rangle+\tan^2(\frac{\theta_q}{2})} \\
    %&\quad\overline{+ \tan^2(\frac{\theta_p}{2}) \tan^2(\frac{\theta_q}{2}) -\tan^2(\frac{\theta_p}{2}) \tan^2(\frac{\theta_q}{2}) \langle w_p,w_q \rangle^2} \\
	    &\quad\sqrt{1 +
            \tan^2(\frac{\theta_p}{2}) 
               +\tan^2(\frac{\theta_q}{2}) + \tan^2(\frac{\theta_p}{2}) \tan^2(\frac{\theta_q}{2}) } \\
               & \quad\overline{-{(1 + \tan(\frac{\theta_p}{2}) \tan(\frac{\theta_q}{2}) \langle w_p,w_q \rangle)}^2}\\
               = & 2|\cos(\frac{\theta_p}{2})\cos(\frac{\theta_q}{2})|\sqrt{%
               (1 + \tan^2(\frac{\theta_p}{2}))(1+ \tan^2(\frac{\theta_q}{2}))} \\ 
               & \quad\overline{-{(1 + \tan(\frac{\theta_p}{2}) \tan(\frac{\theta_q}{2}) \langle w_p,w_q \rangle)}^2}
    \end{align*}

    Finally, recall that $1 + \tan^2(\theta) = \frac{1}{\cos^2(\theta)}$, which gives us
    \begin{align*}
	    &\bar\D(R_{p\bar q}) = 2\sqrt{ 1 }\\
	    &\quad\overline{           -\cos^2(\frac{\theta_p}{2})\cos^2(\frac{\theta_q}{2}){(1 + \tan(\frac{\theta_p}{2}) \tan(\frac{\theta_q}{2}) \langle w_p,w_q \rangle)}^2}
	\end{align*}

   Now we begin to explore the right-hand side of the equation, by noting that when $\sin(\theta)>0$ then $\sin(\theta) = |\sin(\theta)| = \sqrt{\sin^2(\theta)} = \sqrt{1-\cos^2(\theta)}$. Furthermore we note that $\cos^{-1}$ maps $[-1,1]$ to $[0,\pi]$ and particularly it maps $[0,1]$ to $[0,\pi/2]$ where the sine function is positive, therefore, we get

       \begin{align*}
	   & 2\sin(\rho(p,q)) = 2 \sin( \cos^{-1}(|\langle p,q\rangle|)) \\
           & = 2 \sqrt{1 \! - \! \cos^2(\cos^{-1}(|\langle p,q\rangle|))}\\
           & = 2 \sqrt{1 \! - \! \langle p,q\rangle^2}\\
           & = 2 \sqrt{1 \! - \! {(\cos(\!\frac{\theta_p}{2})\cos(\!\frac{\theta_q}{2}) \! + \! \sin(\!\frac{\theta_p}{2})\sin(\!\frac{\theta_q}{2})\langle w_p,w_q\rangle )}^2}\\
       \end{align*}

       And finally, we get the same formula as before:

       \begin{align*}
	       & 2\sin(\rho(p,q)) = 2 \sqrt{1} \\ &\quad\overline{- \cos^2(\!\frac{\theta_p}{2})\cos^2(\!\frac{\theta_q}{2}){(1 +\tan(\!\frac{\theta_p}{2})\tan(\!\frac{\theta_q}{2})\langle w_p,w_q\rangle )}^2}\\
       \end{align*}

       Hence concluding the proof of Proposition~\ref{prop:displacement-norm}.
\end{proof}


\begin{thebibliography}{99}

\bibitem[Aurenhammer et al.\,(1984)]{aurenhammer}
Aurenhammer F and Edelsbrunner H
(1984)
An optimal algorithm for constructing the weighted Voronoi diagram in the plane.
In: \textit{Pattern Recognition}, 17(2), pp.251--257.

\bibitem[Barraquand et al.\,(1997)]{barraquand}
Barraquand J, Kavraki L, Latombe JC, Motwani R, Li TY  and Raghavan P
(1997)
A random sampling scheme for path planning.
In: \textit{The International Journal of Robotics Research}, 16(6), pp.759--774.

\bibitem[Basch et al.\,(2001)]{basch}
Basch J, Guibas LJ, Hsu D and Nguyen AT,
(2001)
Disconnection proofs for motion planning.
In: \textit{IEEE International Conference on Robotics and Automation}, vol. 2, pp. 1765--1772.

\bibitem[Behar and Lien\,(2013)]{behar}
Behar, E and Lien, JM
(2013)
Mapping the configuration space of polygons using reduced convolution
In: \textit{IEEE/RSJ International Conference on Intelligent Robots and Systems}, pp. 1242--1248.

\bibitem[Dey et al.\,(2002)]{dey-medialaxis-mesh}
Dey TK and Zhao W
(2012)
Approximate medial axis as a voronoi subcomplex.
In: \textit{Proceedings of the seventh ACM symposium on Solid modeling and applications}, pp. 356--366.

\bibitem[Dey et al.\,(2006)]{dey-medialaxis-pointcloud}
Dey TK and Sun J
(2006)
Normal and feature approximations from noisy point clouds.
In: \textit{International Conference on Foundations of Software Technology and Theoretical Computer Science} pp. 21--32.

\bibitem[Edelsbrunner\,(1999)]{edelsbrunner-skin}
Edelsbrunner H
(1999)
Deformable smooth surface design.
In: \textit{Discrete \& Computational Geometry}, 21(1), pp. 87--115.

\bibitem[Guibas et al.\,(1985)]{guibas}
Guibas L, and Stolfi J,
(1985)
Primitives for the manipulation of general subdivisions and the computation of Voronoi.
In: \textit{ACM transactions on graphics}, 4(2), pp.74--123.

\bibitem[Kuperberg\,(1990)]{kuperberg}
Kuperberg W
(1990)
Problems on polytopes and convex sets
In: \textit{DIMACS Workshop on polytopes}, pp. 584--589.

\bibitem[Latombe\,(1991)]{latombe}
Latombe JC
(1991)
\textit{Robot motion planning.}
vol. 124. Springer International Series in Engineering and Computer Science. Springer US.

\bibitem[Liu et al.\,(2018)]{icra_caging}
Liu J, Xin S, Gaol Z, Xu K, Tu C and Chen B
(2018)
Caging loops in shape embedding space: theory and computation.
In: \textit{IEEE International Conference on Robotics and Automation (ICRA)}, pp. 1--5. 

\bibitem[Lozano-Perez\,(1983)]{lozano-perez}
Lozano-Perez T
(1990)
Spatial planning: A configuration space approach.
In: \textit{Autonomous robot vehicles}, pp. 259--271. Springer, New York, NY.

\bibitem[Mahler et al.\,(2016)]{mahler}
Mahler J, Pokorny FT, McCarthy Z, van der Stappen AF and Goldberg K
(2016)
Energy-bounded caging: Formal definition and 2-D energy lower bound algorithm based on weighted alpha shapes.
In: \textit{IEEE Robotics and Automation Letters}, 1(1), pp.508--515.

\bibitem[Mahler et al.\,(2018)]{mahler_2}
Mahler J, Pokorny FT, Niyaz S and Goldberg K
(2018)
Synthesis of energy-bounded planar caging grasps using persistent homology.
In: \textit{IEEE Transactions on Automation Science and Engineering}, 15(3), pp.908--918.

\bibitem[Makita and Maeda\,(2008)]{makita}
Makita S and Maeda Y
(2008)
3D multifingered caging: Basic formulation and planning.
In: \textit{ 2008 IEEE/RSJ International Conference on Intelligent Robots and Systems}, pp. 2697--2702.

\bibitem[Makita et al.\,(2013)]{makita2}
Makita S, Okita K and Maeda Y
(2013)
3D two-fingered caging for two types of objects: sufficient conditions and planning.
In: \textit{IJMA}, 3(4), pp.263--277.

\bibitem[Makita and Wan\,(2017)]{makita-survey}
Makita S and Wan W
(2017)
A survey of robotic caging and its applications.
In: \textit{Advanced Robotics}, 31(19--20), pp.1071--1085.

\bibitem[Makapunyo et al.\,(2013)]{makapunyo}
Makapunyo T, Phoka T, Pipattanasomporn P, Niparnan N and Sudsang A
(2013)
Measurement framework of partial cage quality based on probabilistic motion planning.
In: \textit{IEEE International Conference on Robotics and Automation}, pp. 1574--1579.

\bibitem[McCarthy et al.\,(2012)]{mccarthy}
McCarthy Z, Bretl T and Hutchinson S
(2012)
Proving path non-existence using sampling and alpha shapes.
In: \textit{IEEE International Conference on Robotics and Automation}, pp. 2563--2569.

\bibitem[Milenkovic et al.\,(2013)]{milenkovic}
Milenkovic V, Sacks E and Trac S
(2013)
Robust complete path planning in the plane.
In: \textit{Algorithmic Foundations of Robotics X}, pp. 37--52. Springer, Berlin, Heidelberg.

\bibitem[Mitra et al.\,(2013)]{Mitra2013}
Mitra T, Jelfs KE, Schmidtmann M, Ahmed A, Chong SY, Adams DJ and Cooper AI
(2013)
Molecular shape sorting using molecular organic cages
In: \textit{Nature chemistry}, 5(4), p.276.

\bibitem[Pereira et al.\,(2004)]{pereira}
Pereira GA, Campos MF and Kumar V
(2004)
Decentralized algorithms for multi-robot manipulation via caging.
In: \textit{The International Journal of Robotics Research}, 23(7--8), pp.783--795.

\bibitem[Pipattanasomporn and Sudsang\,(2006)]{sudsang_polygons}
Pipattanasomporn P and Sudsang A
(2006)
Two-finger caging of concave polygon.
In: \textit{Proceedings IEEE International Conference on Robotics and Automation} pp. 2137--2142.

\bibitem[Pipattanasomporn and Sudsang\,(2011)]{sudsang_polytopes}
Pipattanasomporn P and Sudsang  A
(2011)
Two-finger caging of nonconvex polytopes.
In: \textit{IEEE Transactions on Robotics}, 27(2), pp.324--333.

\bibitem[Pokorny et al.\,(2013)]{pokorny}
Pokorny FT, Stork JA and Kragic D
(2013)
Grasping objects with holes: A topological approach.
In: \textit{2013 IEEE International Conference on Robotics and Automation}, pp. 1100--1107.

\bibitem[Rimon and Blake\,(1999)]{rimon}
Rimon E and Blake A
(1999)
Caging planar bodies by one-parameter two-fingered gripping systems.
In: \textit{International Journal of Robotics Research}, 18(3), pp.299--318.

\bibitem[Rodriguez and Mason\,(2012)]{rodriguez-path}
Rodriguez A and Mason MT
(2012)
Path connectivity of the free space.
In: \textit{ IEEE Transactions on Robotics}, 28(5), pp.1177--1180.

\bibitem[Rodriguez et al.\,(2012)]{rodriguez}
Rodriguez A, Mason MT and Ferry S
(2012)
From caging to grasping.
In: \textit{The International Journal of Robotics Research}, 31(7), pp.886--900.

\bibitem[Rother et al.\,(2016)]{rother}
Rother M, Nussbaumer MG, Renggli K and Bruns N
(2016)
Protein cages and synthetic polymers: a fruitful symbiosis for drug delivery applications, bionanotechnology and materials science.
In: \textit{Chemical Society Reviews}, 45(22), pp.6213--6249.

\bibitem[Stork et al.\,(2013)]{stork2013b}
Stork JA, Pokorny FT and Kragic D
(2013)
Integrated motion and clasp planning with virtual linking.
In: \textit{ 2013 IEEE/RSJ International Conference on Intelligent Robots and Systems}, pp. 3007--3014.

\bibitem[Vahedi and van der Stappen\,(2008)]{vahedi}
Vahedi M and van der Stappen AF (2008).
Caging polygons with two and three fingers.
In: \textit{The International Journal of Robotics Research}, 27(11--12), pp.1308--1324.

\bibitem[Varava et al.\,(2016)]{varava}
Varava A, Kragic D and Pokorny FT
(2016)
Caging grasps of rigid and partially deformable 3-D objects with double fork and neck features.
In: \textit{IEEE Transactions on Robotics}, 32(6), pp.1479--1497.

\bibitem[Varava et al.\,(2017)]{varava_2}
Varava A, Carvalho JF, Pokorny FT and Kragic D
(2017)
Caging and path non-existence: a deterministic sampling-based verification algorithm.
In: \textit{International Symposium on Robotics Research}

\bibitem[Varava et al.\,(2018)]{varavaWAFR}
Varava A, Carvalho JF, Pokorny FT and Kragic D
(2018)
Free Space of Rigid Objects: Caging, Path Non-Existence, and Narrow Passage Detection.
In: \textit{Workshop on Algorithmic Foundations of Robotics}.

\bibitem[Wang and Kumar\,(2002)]{wang}
Wang Z and Kumar V
(2002)
Object closure and manipulation by multiple cooperating mobile robots.
In: \textit{ Proceedings IEEE International Conference on Robotics and Automation}, pp. 394--399.

\bibitem[Wise and Bowyer\,(2000)]{wise}
Wise KD and Bowyer A
(2000)
A survey of global configuration-space mapping techniques for a single robot in a static environment.
In: \textit{The International Journal of Robotics Research}, 19(8), pp.762--779.

\bibitem[Yershova et al.\,(2009)]{yershova}
Yershova A, Jain S, Lavalle SM and Mitchell JC
(2010)
Generating uniform incremental grids on SO(3) using the Hopf fibration.
In: \textit{The International journal of robotics research}, 29(7), pp.801--812.

\bibitem[Zhang et al.\,(2008)]{zhang}
Zhang L, Kim YJ and Manocha D
(2008)
Efficient cell labelling and path non-existence computation using C-obstacle query.
In: \textit{The International Journal of Robotics Research}, 27(11--12), pp.1246--1257.

\bibitem[Zhu and Latombe\,(1991)]{zhu}
Zhu DJ and Latombe JC
(1991)
New heuristic algorithms for efficient hierarchical path planning.
In: \textit{IEEE Transactions on Robotics and Automation}, 7(1), pp.9--20.

\bibitem[Zomorodian and Edelsbrunner\,(2000)]{zomorodian}
Zomorodian A and Edelsbrunner  H
(2002)
Fast software for box intersections.
In: \textit{International Journal of Computational Geometry \& Applications}, pp.143--172.

\end{thebibliography}
\end{document}